%% file: main.tex
\documentclass{informs3h}
\RequirePackage{bm}

\usepackage{mathtools}
\usepackage{algorithm}
\usepackage{algpseudocode}
\usepackage{booktabs}
\usepackage{multirow}
\usepackage{microtype}
\usepackage{xcolor}
\usepackage{graphicx}
\usepackage{subcaption}
\usepackage[colorlinks=true,linkcolor=black,citecolor=black,urlcolor=black]{hyperref}
\usepackage[nameinlink,capitalise]{cleveref}
\crefname{assumption}{Assumption}{Assumptions}
\Crefname{assumption}{Assumption}{Assumptions}

\usepackage{comment}

\usepackage{natbib}
\bibpunct[, ]{(}{)}{,}{a}{}{,}
\def\bibfont{\small}

\newcommand{\kl}{\operatorname{kl}}

\newcommand{\Pgrid}{\mathcal P}
\newcommand{\logplus}{\log_{+}}

\EquationsNumberedThrough
\TheoremsNumberedThrough

\makeatletter
\def\theARTICLETOP{}
\def\theARTICLEABSTRACT{%
  \vspace*{10pt}%
  \begin{minipage}{\textwidth}\small
    \noindent\textbf{Abstract.}\ \theABSTRACT\par\vspace{4pt}%
    \theKEYWORDS
  \end{minipage}%
  \vspace*{10pt}%
}
\makeatother

\newcommand{\1}{\mathbf{1}}

\definecolor{thmrefcolor}{HTML}{1F4E79}

\DeclareRobustCommand{\thmref}[2]{%
  \hyperref[#1]{%
    \textcolor{thmrefcolor}{#2~\ref*{#1}}%
  }%
}

\algrenewcommand\algorithmicrequire{\textbf{Input:}}
\algrenewcommand\algorithmicensure{\textbf{Output:}}

\begin{document}

\RUNAUTHOR{Anonymous Authors}
\RUNTITLE{Monotone-Adjusted One-Arm Index Policies}
\RRHSecondLine{}
\LRHSecondLine{}
\TITLE{Multi-Armed Bernoulli Bandits via Minimax Single-Arm Stopping}

\ARTICLEAUTHORS{%
\AUTHOR{Huikang Liu\textsuperscript{1},
Zhengchao Wang\textsuperscript{2},
Daniel Kuhn\textsuperscript{3},
Wolfram Wiesemann\textsuperscript{4}}
\AFF{\textsuperscript{1} Shanghai Jiao Tong University;
\textsuperscript{2} The University of Sydney;
\textsuperscript{3} École Polytechnique Fédérale de Lausanne;
\textsuperscript{4} Imperial College Business School}
}

\ABSTRACT{%
We develop an index policy for finite-horizon Bernoulli multi-armed bandits from minimax solutions to single-arm bandit (SAB) problems.
Each SAB problem involves choosing between an unknown Bernoulli arm and a known reward.
We show that minimizing worst-case regret of SAB problems over all non-anticipative policies admits an exact semi-infinite linear programming formulation.
The resulting stopping policies offer a natural way to compare arms: the higher the known reward against which a policy continues sampling, the more promising the unknown arm.
We turn this intuition into indices based on cumulative continuation probabilities, with a monotone adjustment and a reward-shortfall cap.
By relating index errors to the regret of single-arm stopping policies, we establish a distribution-free regret bound of $4.45\sqrt{KT}+10.75K$ for $K$ arms and horizon $T$.
This bound matches the minimax-optimal regret order established in the literature.
The guarantee extends to rewards supported on $[0,1]$ through Bernoulli randomization.
We also provide a finite-grid implementation with quantified approximation loss.
In numerical experiments, the SAB-based index policy achieves lower  worst-case regret than every tested benchmark policy across all evaluated numbers of arms and horizons, while closely matching the grid-based MAB minimax policy in the two-arm setting.
}


\maketitle

\input{main/introduction}

\input{main/literature}
\input{main/section_1_sab}
\input{main/section_2_indices}
\input{main/section_3_mab}

\input{main/numerical_study}

\bibliographystyle{informs2014}
\bibliography{main/references}

\newpage

\begin{APPENDICES}
\input{main/appendix_proof}

\end{APPENDICES}

\end{document}

%% file: main/introduction.tex
\section{Introduction}
\label{sec:introduction}

The stochastic multi-armed bandit (MAB) problem is a fundamental model of sequential decision-making under uncertainty.
Over a horizon of $T$ rounds, a decision maker repeatedly selects one of $K$ arms and observes only the reward from the selected arm.
The arm means are unknown, so each selection both earns a reward and provides information for future decisions.
The objective is to minimize regret, the expected reward lost relative to always selecting an arm with the largest mean.
This exploration--exploitation trade-off has made the MAB problem central to the theoretical study of learning and sequential allocation \citep{Robbins1952,LattimoreSzepesvari2020}, with applications in clinical trials, online experimentation, recommendation systems, revenue management and many more.

The cost of exploration is most visible when the number of decisions is limited.
In a nonprofit direct-mail campaign, for example, an organization may need to choose among several fundraising appeals for a fixed mailing list, a setting motivated by direct-mail fundraising experiments on donor responses \citep{KarlanWood2017}.
If each response is modeled as a donation or no donation, then each appeal is a Bernoulli arm with an unknown response rate.
Earlier responses can guide later mailings, but every exploratory mailing uses a recipient who could have received a better appeal.
Similar finite-opportunity trade-offs arise in small clinical trials and rare-disease studies \citep{villar2015multi,villar2015response}, short-horizon advertising and performance-marketing campaigns \citep{GigliStella2025,geng2021comparison}, online experiments where experimentation traffic is limited or costly \citep{scott2015multi}, and operational learning problems in revenue management and service systems \citep{besbes2009dynamic,ferreira2018online,krishnasamy2021learning,choudhury2021job}.
In such problems, exploration is not a negligible transient cost: committing too early may abandon the best alternative, while exploring too long may leave too few opportunities to benefit from what has been learned.

The choice of performance criterion is therefore central.
One common approach is Bayesian: specify a prior over the unknown arm means and evaluate a policy by its average performance under that prior.
Another approach is minimax: evaluate a policy by its largest expected regret over all possible arm means in the problem class.
The minimax criterion is appealing when the decision maker wants a guarantee that does not depend on a correctly specified prior.
It gives a form of protection: no matter which mean vector is realized, the expected reward loss relative to the best arm is controlled by the stated worst-case bound.
This is the perspective we take in this paper.

Many successful bandit policies were designed from different principles.
UCB-type policies use concentration inequalities to build optimistic estimates, Thompson sampling uses posterior randomization, and Gittins-type indices are optimal in classical Bayesian discounted settings.
These ideas have led to powerful finite-time guarantees, including policies with the optimal $\sqrt{KT}$ minimax regret order.
However, attaining the optimal order does not identify the policy that minimizes worst-case regret at a fixed horizon.
Two policies with the same order can still spend their early exploration very differently and incur different finite-horizon losses.
Thus, the available regret theory leaves open a more direct design question: 
\begin{center}
    \boxed{ \emph{Can we build a practical MAB policy from the finite-horizon minimax objective itself?}}
\end{center}

We do so by starting with the simpler single-arm bandit (SAB) problem.
In the SAB problem, the decision maker compares an unknown Bernoulli arm with a known arm that gives a deterministic reward $c\in[0,1]$.
The known reward serves as a benchmark: after observing a history from the unknown arm, the policy decides whether the arm is still worth sampling when the alternative is the fixed reward $c$.
For this problem, stopping policies are sufficient: the policy samples the unknown arm until it stops, and then selects the known arm for the remaining rounds.
We optimize these stopping decisions for worst-case regret over the unknown Bernoulli mean.
This gives a finite-horizon minimax comparison between an uncertain arm and a known reward, which is exactly the comparison we later use to construct MAB indices.

To obtain these stopping policies, we develop an exact semi-infinite linear programming formulation of the SAB minimax stopping problem.
The formulation optimizes over continuation probabilities and recovers a minimax stopping policy from an optimal solution.
We also give an exact occupation-measure formulation for the MAB minimax problem.
That formulation is useful conceptually and in small numerical instances, but its size grows quickly with the number of arms.
By contrast, the SAB formulation is much smaller and admits a practical finite-grid approximation.
This computational contrast motivates us to use the SAB minimax problem as a tractable building block to construct a MAB policy, rather than solving the MAB minimax formulation directly.

The key step in this construction is to use the known-arm reward \(c\) as a common scale for comparing different histories of the unknown arm. 
For a given history, solving the SAB minimax problem for different values of \(c\) allows us to identify the largest $c$ for which the unknown arm still merits further sampling. 
In this sense, an arm history associated with a larger value of \(c\) is more promising.
This suggests a natural index: assign each arm the largest known-arm reward \(c\) for which the SAB policy would continue sampling after that arm's observed history.
The MAB policy shall then select an arm with the largest such index.
The formal construction requires some care because the raw SAB continuation decisions need not have the monotonicity one would ideally want.
Those technical issues are handled in the index construction, while the high-level meaning remains simple: each arm is ranked by the largest known-arm reward against which its current history still justifies further sampling.
We refer to the resulting policy as the SAB-based index policy.

In addition, a SAB minimax policy does not only tell us whether a history is worth sampling against a known-arm reward \(c\); its regret also measures the cost of making that comparison incorrectly.
This observation is what allows the one-arm analysis to enter the MAB proof.
If the index of a promising arm is too low, the MAB policy may overlook that arm.
If the index of an inferior arm is too high, the policy may spend too many pulls on it.
We show that these two index errors can be traced back to the same mistakes measured by the auxiliary SAB regret: stopping too early relative to a known-arm reward, or continuing too long relative to a known-arm reward.
Thus, the regret guarantees of the SAB stopping policies become a way to control the errors made by the multi-arm index policy.

This connection leads to an explicit distribution-free worst-case regret bound of \(4.45\sqrt{KT}+10.75K\) for \(T\geq 2K\).
The bound has the minimax-optimal dependence on the horizon and the number of arms, matching the \(\sqrt{KT}\) order in the classical lower bound of \citet{auer2002nonstochastic}.
Although the policy is constructed for Bernoulli rewards, the guarantee extends to arbitrary reward distributions supported on \([0,1]\) through Bernoulli randomization.
Although the leading constant in our regret bound is slightly larger than the smallest known value $2$ of \cite{ZimmertSeldin2021}, numerical results show that our proposed algorithm nevertheless achieves much lower worst-case regret than the algorithm in \cite{ZimmertSeldin2021}.

Through the MAB regret analysis, it does not require every auxiliary SAB policy to be exactly minimax optimal.
It only requires a suitable uniform worst-case regret bound for the SAB stopping policy family used to construct the index.
As it will become clear later, this flexibility makes it possible to construct the index over a discretizations of $c \in [0, 1]$ and to solve the SAB semi-infinite linear program via discretizations over the unknown-arm reward for each discrete $c$. 
Because the auxiliary SAB policies depend only on the horizon and the known-reward grid, they can be solved before the MAB process begins.
During sampling, the policy only updates the relevant precomputed continuation values and compares the current arm indices.

We also present numerical experiments to assess the finite-horizon performance of this SAB minimax index construction.
In the SAB experiments, the grid-based SAB minimax policy achieves worst-case regret approximately half that of the best-performing benchmark across all tested grids and horizons up to $T=200$.
The numerical results show that our proposed SAB policy outperforms the benchmark policies by more precisely limiting unnecessary exploration early in the horizon while permitting necessary exploration later.
In the MAB experiments, the SAB-based index policy achieves substantially lower estimated worst-case regret than every benchmark policy across all tested horizons and numbers of arms.
For the two-arm setting, we also solve a finite-grid discretization of the MAB minimax formulation.
The SAB-based index policy remains close in performance to this directly optimized policy across all tested horizons.
These results suggest that, in the two-arm setting, the SAB-based index policy retains much of the benefit of direct MAB minimax optimization while avoiding the need to solve exactly the MAB program.

The remainder of the paper is organized as follows.
Section~\ref{sec:sab} studies the SAB minimax problem and its numerical approximation.
Section~\ref{sec:indices} constructs the adjusted SAB index.
Section~\ref{sec:mab} defines the MAB policy, proves the regret bound, and explains the finite implementation.
Section~\ref{sec:experiments} reports the numerical experiments.
The appendices contain the proofs and the verification of the numerical constants.

%% file: main/literature.tex
\section{Literature review}

We discuss the related literature from two perspectives: how bandit policies are constructed and the criteria under which they are optimal or near-optimal.
Our focus is on one-arm stopping problems as a basis for policy construction and on the distinction between minimax-order regret guarantees and direct finite-horizon minimax optimization.
For broader treatments of bandit algorithms and regret analysis, we refer to \citet{BubeckCesaBianchi2012} and \citet{LattimoreSzepesvari2020}.

\subsection{Policy construction.}
A basic approach to exploration is to separate it from subsequent exploitation.
Explore-then-commit policies first collect observations and then select the arm judged best for the remaining rounds \citep{GarivierLattimoreKaufmann2016}.
The double explore-then-commit construction of \citet{JinXuXiaoGu2021} allows additional exploration after an initial exploitation phase.
UCB-type policies instead continually compare optimistic estimates of the arm means, combining observed rewards with an exploration bonus that accounts for uncertainty \citep{AuerCesaBianchiFischer2002}.
For Bernoulli rewards, KL-UCB refines this comparison using confidence sets based on Bernoulli Kullback--Leibler divergence \citep{GarivierCappe2011}.
Whereas these approaches determine exploration through phased sampling or optimistic estimates, our starting point is to optimize the sampling policy in SAB problems for finite-horizon worst-case regret.

Posterior-based policies determine sampling decisions through beliefs about the unknown arm means.
Thompson sampling draws a parameter vector from the posterior distribution and selects an arm with the largest sampled mean \citep{Thompson1933}.
Other posterior-based constructions account more explicitly for the information obtained from sampling.
Information-directed sampling balances immediate expected regret against information gained about the optimal arm \citep{RussoVanRoy2018}.
For finite-horizon Bayesian bandits, \citet{MinMaglarasMoallemi2025} develop information-relaxation sampling policies that optimize against sampled future rewards, with penalties for using information not yet observed.
This approach incorporates the remaining horizon into the sampling decision.
These methods use a specified prior to evaluate uncertain rewards, whereas our SAB problems optimize directly worst-case regret over the unknown mean.

A closely related approach to ours is to evaluate each arm through a one-arm stopping problem.
Gittins indices compare continued sampling of an arm with a known reward, using the reward at which the decision maker is indifferent as an index \citep{Gittins1979,Whittle1980}.
This comparison allows arms to be ranked through separate one-arm calculations.
\citet{NinoMora2011} develops an exact algorithm for computing finite-horizon stopping-based indices and compares it with methods that solve stopping problems over a grid of known-arm rewards.
For the classical infinite-horizon discounted MAB problem, \citet{BertsimasNinoMora1996} use a linear programming formulation of expected-reward maximization to derive the optimal Gittins-index policy.
Our construction also uses one-arm optimization, but the stopping policies minimize worst-case regret over a finite horizon.
We obtain these policies from exact SAB semi-infinite linear programs and use them to construct the SAB-based index policy.
This provides a way to design a MAB policy from the finite-horizon minimax objective without solving the MAB minimax problem.

\subsection{Optimality and regret guarantees.}
In the Bayesian setting, an optimal policy maximizes expected, possibly discounted, total reward under a specified prior.
\citet{BradtJohnsonKarlin1956} study finite-horizon sequential allocation, including the comparison between an unknown arm and a known arm.
For the undiscounted known-arm problem, an optimal policy samples the unknown arm until stopping and then selects the known arm for all remaining rounds.
\citet{Jones1978} studies this problem under a general prior, while \citet{BerryFristedt1979} consider arbitrary discount sequences.
\citet{BerryFristedt1985} provide a broader treatment of optimal policies for one-arm and multi-arm problems.
For independent arms, an infinite horizon, and geometric discounting, the Gittins theorem establishes that selecting an arm with the largest index is optimal \citep{Gittins1979,GittinsGlazebrookWeber2011}.
At a fixed, undiscounted horizon, however, optimal one-arm stopping calculations do not generally yield an optimal Bayesian MAB policy \citep{NinoMora2011}.
We likewise distinguish exact minimax optimality of the SAB policies from the worst-case regret guarantee for the resulting MAB policy.

Frequentist regret analysis evaluates policies at fixed unknown arm means rather than averaging over a specified prior.
For fixed instances with positive reward gaps, \citet{LaiRobbins1985} establish an asymptotic lower bound that relates the necessary exploration to the statistical difficulty of distinguishing the arms.
\citet{AuerCesaBianchiFischer2002} give finite-time logarithmic regret bounds for UCB, and KL-UCB attains the asymptotic lower bound for Bernoulli rewards \citep{GarivierCappe2011}.
The minimax criterion instead requires control uniformly over instances, including those whose reward gaps shrink with the horizon.
MOSS achieves the $\sqrt{KT}$ minimax regret order for $T\ge K$, matching the classical lower bound up to constants \citep{AudibertBubeck2009,AuerCesaBianchiFreundSchapire2002}.
Later policies combine minimax-order regret with fixed-instance asymptotic optimality, including KL-UCB++ for one-parameter exponential families and KL-UCB-Switch for rewards supported on $[0,1]$ \citep{MenardGarivier2017,GarivierHadijiMenardStoltz2022}.
These guarantees establish the optimal regret order, but do not identify the policy minimizing worst-case regret at a specified horizon.
Our paper finds an exact minimax optimal policy for the SAB problem, while attaining the established minimax order for the MAB problem with our SAB-based index policy.

Frequentist guarantees are also available for other policy constructions.
For Thompson sampling, \citet{AgrawalGoyal2012} establish finite-time regret bounds for Bernoulli bandits, and \citet{JinYangXiaoXu2023} show that reduced exploration can yield both minimax-order and fixed-instance asymptotic optimality for several reward families.
Tsallis-INF uses regularized optimization of estimated losses to achieve minimax-order worst-case guarantees that also apply to stochastic bandits \citep{ZimmertSeldin2021}.
For stopping-based indices, \citet{Lattimore2016} establishes frequentist regret bounds for a finite-horizon Gittins policy with Gaussian rewards, including a distribution-free bound with leading term of order $\sqrt{KT\log T}$.
\citet{FariasGutin2022} develop optimistic Gittins indices that attain the Lai--Robbins lower bound asymptotically.
Our analysis relates the regret of the MAB policy to the regret of its SAB stopping policies, yielding a nonasymptotic bound of order $\sqrt{KT}$ for Bernoulli bandits.

A related literature studies optimal policies and minimal risk through asymptotic analysis.
For two unknown Bernoulli arms, \citet{Vogel1960} establishes asymptotic bounds on minimax regret, and \citet{Bather1983} sharpens the asymptotic lower bound.
For a Gaussian bandit with one known arm, \citet{Kolnogorov2019} derives Bayesian dynamic programming recursions and uses a diffusion approximation to study minimax regret.
\citet{Adusumilli2025} characterizes minimal Bayes risk under local asymptotics and numerically studies minimax policies in the limiting problem.
Both papers use least-favorable priors, chosen to maximize the optimal Bayes risk, to study the minimax criterion.
Related work by \citet{KuangWager2024} uses diffusion limits to analyze the regret of sequentially randomized policies when reward gaps shrink at rate $T^{-1/2}$.
We share this interest in exact minimax performance beyond the regret order, but optimize SAB stopping policies directly at the specified finite horizon.
Our exact semi-infinite linear program characterizes the minimax value over the whole Bernoulli parameter space, and we bound the additional regret from finite-grid approximation in the original finite-horizon problem.

Finite-horizon minimax optimization also has direct predecessors in Bernoulli bandits.
\citet{FabiusVanZwet1970} study two unknown Bernoulli arms and compute minimax policies for horizons up to four.
They also show that decisions can be based on each arm's pull and success counts, a reduction used in our formulations.
\citet{Kolnogorov2014} studies finite-horizon Bernoulli MAB problems over arbitrary parameter sets and reduces minimax-risk computation to a search over finitely supported least-favorable priors.
He represents randomized policies using linear constraints and reports numerical minimax values for problems with two unknown arms and with one known arm.
\citet{KolnogorovGrunev2021} further study minimax policies for Bernoulli two-armed bandits at moderate horizons.
For the known-arm problem, our semi-infinite linear program both characterizes the minimax value and recovers a minimizing stopping policy, with explicit regret bounds for finite-grid approximations.
The formulation therefore supplies computable SAB policies whose approximation error can be accounted for in the subsequent MAB guarantee.

Our exact MAB semi-infinite linear formulation characterizes the full minimax problem. 
Although the number of decision variables grows rapidly with the number of arms, we demonstrate that, using an appropriate discretization that yields a small optimality gap, our formulation can solve in one hour the two-arm problem for at least $T=200$ and can accommodate even longer horizons when additional computation time is allowed.
We use the smaller SAB problems to construct the index policy rather than solve the MAB formulation directly for larger problem instance.
Because the regret analysis permits controlled approximation of the SAB stopping policies, solving finitely many SAB programs suffices to obtain an implementable MAB policy that achieves minimax-order regret.

%% file: main/section_1_sab.tex
\section{The Single-Armed Bandit Problem and Its Minimax Policy}
\label{sec:sab}

We begin with a single-armed bandit (SAB) problem that will later be used to construct indices for the multi-armed bandit (MAB) problem.
In the SAB problem, the decision maker chooses between an unknown Bernoulli arm with mean $p\in[0,1]$ and a known arm that gives the deterministic reward $c\in[0,1]$.
We call $c$ the \emph{known-arm reward} throughout the paper.
The unknown mean is fixed but unavailable to the policy, and the objective is to minimize worst-case regret over all possible values of that mean.
A policy attaining this value is called a minimax policy.
This section characterizes the SAB minimax policy through a semi-infinite linear program and gives explicit upper and lower bounds on its regret.
The continuation probabilities obtained from the program are used in Section~\ref{sec:indices} to define arm indices.
The explicit upper bound is the only property of SAB optimality used in the MAB regret guarantee.
Section~\ref{sec:mab} specifies the MAB policy, explains its regret analysis, and distinguishes exact index evaluation from numerical approximation.
All proofs are collected in Appendix~\ref{app:proofs}.

Throughout, a tilde denotes a random variable, and the same lowercase symbol without a tilde denotes its realization.
For example, $\tilde z_j$ is a reward and $z_j$ is its observed value.
Probabilities and expectations include both the rewards and any independent randomization used by the policy.
An admissible policy uses only previously observed rewards and its own randomization.
Let $G\ge1$ be the integer SAB horizon, and let $\tilde n^{\pi}_{c,G}$ be the total number of unknown-arm pulls under a policy $\pi$.
Its regret and the minimax value are
\begin{equation}
\label{eq:sab-regret}
R_{c,G}(\pi,p)
=G(p-c)_+-(p-c)\mathbb E_p[\tilde n^{\pi}_{c,G}],
\qquad
V_G(c):=\inf_{\pi}\sup_{p\in[0,1]}R_{c,G}(\pi,p),
\end{equation}
where $(v)_+:=\max\{v,0\}$ and the infimum is over all admissible policies.
Thus regret is $(c-p)$ times the expected number of unknown-arm pulls when $p<c$, and $(p-c)$ times the expected number of known-arm pulls when $p>c$.
Write $\tilde s_n:=\sum_{j=1}^n\tilde z_j$ for the number of successes among the first $n$ unknown-arm observations, with $\tilde s_0=0$.

\subsection{An exact semi-infinite linear program}
\label{subsec:sab-lp}
In the SAB literature, \citep{BradtJohnsonKarlin1956} have shown that it is sufficient to consider only stopping policies to find an optimal policy. 
A stopping policy pulls the unknown arm until it stops and then selects the known arm for all remaining rounds.
For Bernoulli rewards, it is enough to use the count state $(n,\tilde s_n)$ rather than the full ordered history, where $n$ is the number of unknown-arm observations and $\tilde s_n$ is the number of successes.
The following proposition states the reduction required here.

\begin{proposition}
\label{prop:count-state}
For every integer \(G\ge 1\) and known-arm reward \(c\in[0,1]\), it suffices to consider randomized stopping policies that depend only on \((n,s_n)\): every admissible stopping policy has such a counterpart with identical regret for all \(p\in[0,1]\).
\end{proposition}

If the stopping policy is in state $(n,s)$, let $q^{c,G}_{n,s}\in[0,1]$ be its probability of taking another unknown-arm observation.
These continuation probabilities are specified for every formal state, including states that the policy never reaches.
Let $\tilde a^{c,G}_t\in\{0,1\}$ be the action in SAB round $t$, where $1$ denotes the unknown arm and $0$ denotes the known arm.
Because the policy is a stopping policy, taking the $(n+1)$st unknown-arm observation requires selecting the unknown arm in each of the first $n$ rounds and then continuing once more.
To express this event, define
\[
E^{c,G}_{n,s}:=\{\tilde a^{c,G}_1=\cdots=\tilde a^{c,G}_n=1,\ \tilde s_n=s\},
\qquad
w^{c,G}_{n,s}(p):=\mathbb P_p(E^{c,G}_{n,s},\tilde a^{c,G}_{n+1}=1),
\]
with $E^{c,G}_{0,0}$ equal to the whole sample space.
Thus $w^{c,G}_{n,s}(p)=q^{c,G}_{n,s}\mathbb P_p(E^{c,G}_{n,s})$ is the probability of reaching state $(n,s)$ and continuing.
Its dependence on $p$ comes from the Bernoulli observations.

For $n=0$, $w^{c,G}_{0,0}(p)=q^{c,G}_{0,0}$.
For $1\le n<G$, state $(n,s)$ can be reached by a success from $(n-1,s-1)$ or a failure from $(n-1,s)$.
Therefore,
\[
w^{c,G}_{n,s}(p)
=q^{c,G}_{n,s}\bigl[pw^{c,G}_{n-1,s-1}(p)+(1-p)w^{c,G}_{n-1,s}(p)\bigr],
\]
where terms outside the natural index ranges are zero.
For $p\in(0,1)$, define the normalized continuation probability $x^{c,G}_{n,s}$ by
\(w^{c,G}_{n,s}(p)=\binom ns p^s(1-p)^{n-s}x^{c,G}_{n,s}.\)
This normalization removes the binomial probability of observing $s$ successes in $n$ observations, as the recursion below shows that $x^{c,G}_{n,s}$ does not depend on $p$.
It is not the same as the conditional continuation probability $q^{c,G}_{n,s}$.
For $p=0$ and $p=1$, the binomial factor is interpreted by continuous extension: it equals one at $(p,s)=(0,0)$ and $(p,s)=(1,n)$, and zero at impossible states.

Substituting the factorization into the recursion and using the ratios of consecutive binomial coefficients gives
\[
nx^{c,G}_{n,s}
=q^{c,G}_{n,s}\bigl[sx^{c,G}_{n-1,s-1}+(n-s)x^{c,G}_{n-1,s}\bigr].
\]
Since $q^{c,G}_{n,s}\in[0,1]$, every stopping policy satisfies the linear constraints
\[
0\le x^{c,G}_{n,s}\le1,
\qquad
nx^{c,G}_{n,s}\le sx^{c,G}_{n-1,s-1}+(n-s)x^{c,G}_{n-1,s}.
\]
Conversely, an array satisfying these constraints defines a stopping policy.
Set $q^{c,G}_{0,0}:=x^{c,G}_{0,0}$ and, for $1\le n<G$, use the ratio of the left side of the flow constraint to its right side whenever the denominator is positive.
If the denominator is zero, the constraint forces $x^{c,G}_{n,s}=0$; that state cannot be reached while sampling continues, and we set $q^{c,G}_{n,s}=0$.
Thus the constraints give an exact linear representation of stopping policies, not a relaxation.

The expected number of unknown-arm pulls is the sum of the probabilities of taking the first, second, and subsequent observations.
Partitioning each such event by its success count gives
\[
\mathbb E_p[\tilde n^{\pi}_{c,G}]
=\sum_{n=0}^{G-1}\sum_{s=0}^n w^{c,G}_{n,s}(p)
=\sum_{n=0}^{G-1}\sum_{s=0}^n\binom ns p^s(1-p)^{n-s}x^{c,G}_{n,s}.
\]
Substitution into \eqref{eq:sab-regret} yields the following formulation.

\begin{theorem}
\label{thm:sab-lp}
For every  $G\ge1$ and known-arm reward $c\in[0,1]$, $V_G(c)$ is the optimal value of
\begin{equation}\label{eq:sab-lp}
\begin{aligned}
\min_{\eta,x^{c,G}}\quad &\eta\\
\text{subject to}\quad
&\eta\ge G(p-c)_+-(p-c)\sum_{n=0}^{G-1}\sum_{s=0}^n
\binom ns p^s(1-p)^{n-s}x^{c,G}_{n,s},
&&p\in[0,1],\\
&0\le x^{c,G}_{n,s}\le1,
&&0\le s\le n<G,\\
&nx^{c,G}_{n,s}\le sx^{c,G}_{n-1,s-1}+(n-s)x^{c,G}_{n-1,s},
&&1\le n<G,\ 0\le s\le n.
\end{aligned}
\end{equation}
Terms outside the natural index ranges are zero.
An optimal solution exists.
Any minimizing array $x^{\star,c,G}$ defines a minimax stopping policy $\pi^{\star}_{c,G}$ by $q^{\star,c,G}_{0,0}:=x^{\star,c,G}_{0,0}$ and
for $1\le n<G$ and $0\le s\le n$
\begin{equation}
\label{eq:sab-recovery}
q^{\star,c,G}_{n,s}:=
\begin{cases}
\displaystyle\frac{nx^{\star,c,G}_{n,s}}
{sx^{\star,c,G}_{n-1,s-1}+(n-s)x^{\star,c,G}_{n-1,s}},
&sx^{\star,c,G}_{n-1,s-1}+(n-s)x^{\star,c,G}_{n-1,s}>0,\\[1ex]
0,&\text{otherwise}.
\end{cases}
\end{equation}
\end{theorem}

The variable $\eta$ bounds worst-case regret from above.
The program is semi-infinite because it has $G(G+1)/2$ policy variables but one regret constraint for each $p\in[0,1]$.
For fixed $p$, every constraint is linear in the decision variables.
At special case $c=0$, every minimizing array satisfies $x^{\star,0,G}_{n,s}=1$, and the recovered policy has $q^{\star,0,G}_{n,s}=1$ at every state.
At $c=1$, every minimizing array satisfies $x^{\star,1,G}_{n,s}=0$, and the recovered policy has $q^{\star,1,G}_{n,s}=0$ at every state.
These policies always select the unknown arm and the known arm, respectively, and give $V_G(0)=V_G(1)=0$.
The proof in Appendix~\ref{app:sab-representation} establishes both directions of the policy representation and uses compactness to establish attainment.

\subsection{Upper bounds and numerical approximation of the minmax value}
\label{subsec:sab-bounds}

The semi-infinite program in Theorem~\ref{thm:sab-lp} gives an exact characterization of the SAB minimax policy.
For the later MAB analysis, however, we do not need all details of this policy.
We only need a uniform upper bound on the SAB minimax regret.
This subsection first gives such a bound.
It then explains how much regret is added when the semi-infinite program is solved through a finite grid of unknown-arm means.

\begin{proposition}
\label{cor:sab-upper}
For every integer $G\ge1$ and known-arm reward $c\in[0,1]$,
\begin{equation}
\label{eq:sab-upper}
\sup_{p\in[0,1]}R_{c,G}(\pi^{\star}_{c,G},p)
=V_G(c)
\le\sqrt{\frac{G}{8e}}+c.
\end{equation}
\end{proposition}

The proof of the result is in Appendix~\ref{app:sab-bounds}. 
This bound is used later only through the SAB regret constant that enters the MAB regret analysis.

We next discuss numerical approximation of the semi-infinite program.
The program in Theorem~\ref{thm:sab-lp} has finitely many policy variables but one regret constraint for every $p\in[0,1]$.
In computation, it is natural to impose these regret constraints only on a finite grid.
For an integer $M\ge1$, let
\(\mathcal P_M:=\left\{0,\frac1M,\frac2M,\ldots,1\right\}.\)
Let $V_{G,M}(c)$ be the optimal value of the resulting finite-grid program, obtained by replacing the constraint indexed by all $p\in[0,1]$ in \eqref{eq:sab-lp} with constraints only at $p\in\mathcal P_M$.
Let $\pi_{c,G,M}$ be the stopping policy reconstructed from an optimal finite-grid solution by the recovery formula in \eqref{eq:sab-recovery}.
The following result bounds the loss from this approximation.

\begin{proposition}
\label{prop:sab-grid}
For integers $G,M\ge1$ and known-arm reward $c\in[0,1]$, the finite-grid program attains its optimum and satisfies
\[
0\le V_G(c)-V_{G,M}(c)\le\frac{G(G+1)}{4M}.
\]
The stopping policy $\pi_{c,G,M}$ reconstructed from any minimizing array satisfies
\[
\sup_{p\in[0,1]}R_{c,G}(\pi_{c,G,M},p)
\le V_{G,M}(c)+\frac{G(G+1)}{4M}.
\]
In addition, the regret $R_{c,G}(\pi,p)$ is $G$-Lipschitz continuous w.r.t. $c$ for any policy $\pi$ that is invariant with $c$, i.e.,
\begin{equation}
\label{eq:sab-c-lipschitz}
|R_{c,G}(\pi,p)-R_{c',G}(\pi,p)|
\le G|c-c'|,
\qquad \forall c,c',p\in[0,1].
\end{equation}
Consequently, $|V_G(c)-V_G(c')|\le G|c-c'|$ for all $c,c'\in[0,1]$.
\end{proposition}

The first inequality shows that the finite-grid value is a lower bound on the exact minimax value.
This is expected because the finite-grid program enforces fewer regret constraints than the semi-infinite program.
The second inequality shows that the policy obtained from the finite-grid program remains feasible for the original SAB problem up to an explicit approximation error.
In particular, choosing $M$ on the order of $G^2$ makes this additional regret uniformly bounded as $G$ grows.

The Lipschitz continuity of the regret $R_{c,G}(\pi,p)$ w.r.t. $c$ addresses a separate numerical issue.
It allows a regret bound verified at finitely many known-arm rewards to be extended to all $c\in[0,1]$.
This is useful when certifying a uniform SAB regret constant over the continuum of known-arm rewards.
The fixed-policy bound~\eqref{eq:sab-c-lipschitz} also permits a finite
representation of the auxiliary family over known-arm rewards.
A policy computed at a grid point $v$ can be reused at a nearby
reward $c$, with at most $G|c-v|$ additional regret.
Section~\ref{subsec:implementation} uses this observation to define
a piecewise-constant auxiliary family from finitely many solved
SAB programs and to evaluate its index exactly.
This construction does not assume continuity of the optimal
continuation probabilities or approximate the index of a
preselected continuum family.

%% file: main/section_2_indices.tex
\section{From Single-Armed Stopping Policies to Multi-Armed Indices}
\label{sec:indices}

An exact LP formulation for the minimax-optimal \(K\)-armed policy is given in Appendix~\ref{subsec:mab-exact}.
As in Section~\ref{sec:sab}, for Bernoulli arms the ordered observation history can be compressed into the vector of pull and success counts for the \(K\) arms.
This leads to the semi-infinite LP in \eqref{eq:mab-lp}.
For fixed \(K\), however, the number of variables grows as \(\Theta(G^{2K})\), so the formulation is computationally practical only for small \(K\) or short horizons \(G\).

We therefore develop an index policy using the SAB policies computed in Section~\ref{sec:sab} for the MAB problem.
Consider the history \(h_n=(z_1,\ldots,z_n)\) of one arm and treat that arm as the unknown arm in an auxiliary SAB problem with known-arm reward \(c\).
For each \(c\), the SAB policy determines whether continuing pulling the unknown arm after \(h_n\).
The largest value of \(c\) for which continuation remains desirable therefore provides a natural index: an arm receives a larger index when its observed history justifies further pulling against a more rewarding known arm.
The known arm is only an auxiliary device for constructing the index; the MAB policy itself selects the arm with the largest index among the \(K\) unknown arms.

To make this construction precise, we introduce cumulative continuation probabilities, defined in the next subsection, to summarize how strongly the policy favors further sampling along the observed history.
The main advantage of defining the index through this cumulative continuation probability is that it preserves the meaning of the underlying SAB decisions.
As shown below, an index that is too small corresponds to stopping when the unknown arm is better than the auxiliary known arm, whereas an index that is too large corresponds to continuing when it is worse.
These are exactly the two types of decisions that generate regret in the auxiliary SAB problem.
The SAB regret bound can therefore be used to control the errors of the resulting MAB index.

A remaining difficulty is that,  somewhat counterintuitively, the continuation behavior of the SAB policies need not be monotone in the known-arm reward \(c\).
We therefore adjust the family so that continuation at a larger known-arm reward implies continuation at every smaller reward, while retaining the required regret control.
This monotonicity gives the threshold structure needed to define the index.
We then bound the underestimation and overestimation of the resulting indices and use these bounds in Section~\ref{sec:mab} to establish the regret guarantee for the MAB policy.

\subsection{Cumulative continuation probabilities and regret}
\label{subsec:continuation}

Fix $c\in[0,1]$ and \emph{any} SAB stopping policy $\pi_{c,G}$, which may depend on the full ordered reward history.
For $h_n=(z_1,\ldots,z_n)$ with $0\le n<G$, let $q^{c,G}(h_n)$ be the probability of taking observation $n+1$, conditional on having reached this history.
At histories that cannot be reached under the policy, set $q^{c,G}(h_n)=0$.
Let $h_j=(z_1,\ldots,z_j)$ be the first $j$ observations of $h_n$, with $h_0$ the empty history, and define the cumulative continuation probability as 
\[
\rho^{c,G}(h_n) \coloneqq \prod_{j=0}^n q^{c,G}(h_j),\qquad 0\le n<G.
\]
The continuation probability $q^{c,G}(h_n)$  only concerns the current decision while $\rho^{c,G}(h_n)$ includes all the decisions required to take observation $n+1$.
For example, $\rho^{c,G}(h_1)=q^{c,G}(h_0)q^{c,G}(h_1)$ includes both the decision to take the first observation and the decision to take the second.
Adding an observation adds a factor in $[0,1]$, so cumulative continuation probabilities are nonincreasing along a fixed reward sequence.


Draw one threshold $\tilde u\sim\operatorname{Unif}(0,1)$ at the beginning independently of the entire reward sequence $\tilde h_n=(\tilde z_1,\ldots,\tilde z_n)$.
Define
\[
\tilde a^{c,G}_{n+1} \coloneqq \mathbf1\{\tilde u<\rho^{c,G}(\tilde h_n)\},
\qquad
\tilde n^{\mathrm{shared}}_{c,G} \coloneqq \sum_{n=0}^{G-1}\tilde a^{c,G}_{n+1}.
\]
Here $n$ is the number of observations already available, and $\tilde a^{c,G}_{n+1}=1$ means that observation $n+1$ is taken.
The same threshold is used at every count.
Once one indicator is zero, since $\rho^{c,G}(\tilde h_n)$ is nonincreasing, all later indicators are zero; consequently, these indicators define a stopping policy.
The next proposition relates this shared-threshold construction to the regret of the original stopping policy.

\begin{proposition}
\label{prop:shared-threshold}
Fix $c\in[0,1]$ and an integer $G\ge1$.
Let $\mathcal Z \coloneqq \sigma(\tilde z_1,\tilde z_2,\ldots)$, and let $\tilde n^{\mathrm{ind}}_{c,G}$ be the pull count obtained with a new independent random draw at each continuation decision.
For every $p\in[0,1]$ and $0\le n<G$,
\[
\mathbb P_p(\tilde n^{\mathrm{ind}}_{c,G}\ge n+1\mid\mathcal Z)
=\mathbb P_p(\tilde n^{\mathrm{shared}}_{c,G}\ge n+1\mid\mathcal Z)
=\rho^{c,G}(\tilde h_n).
\]
Both constructions have the same conditional pull-count distribution and the same regret as $\pi_{c,G}$.
\begin{equation}
\label{eq:continuation-regret}
R_{c,G}(\pi_{c,G},p)
=(c-p)^+ \sum_{n=0}^{G-1}\mathbb E_p\!\bigl[\rho^{c,G}(\tilde h_n)\bigr] + (p-c)^+ \sum_{n=0}^{G-1}\mathbb E_p\!\bigl[1 - \rho^{c,G}(\tilde h_n)\bigr]
\end{equation}
\end{proposition}


Proposition~\ref{prop:shared-threshold} rewrites the regret of a stopping policy as a weighted sum of wrong continuation decisions at benchmark $c$.
If $p>c$, the unknown arm is better than the benchmark, so the correct decision is to continue; an error occurs when $\tilde a_n^{c,G}=0$.
If $p<c$, the unknown arm is worse than the benchmark, so the correct decision is to stop; an error occurs when $\tilde a_n^{c,G}=1$.
Thus, for $p\ne c$, after division by $|p-c|$, a regret bound for any stopping policy of SAB problem gives an upper bound on the total probability of wrong continuation decisions across observations.
This is why we use $\rho^{c,G}$, rather than the one-step probability $q^{c,G}$, to construct the index: $\rho^{c,G}$ is the probability that appears in the regret decomposition.

\subsection{Monotone adjustment and the reward-shortfall policy}
\label{subsec:adjustment}

We note that, for a fix history $h_n$, $\rho^{c,G}(h_n)$ is actually not nonincreasing in $c$\footnote{We can show this by example but we skip here.}.
This can make the index hard to analyze, because the later MAB proof needs the following implication: if an arm has index above $c$, then the associated SAB policy continues at $c$.
We obtain this implication by replacing $\rho^{c,G}$ with the monotone envelope
\[
\bar\rho^{c,G}(h_n) \coloneqq \sup_{v\in[c,1]}\rho^{v,G}(h_n).
\]
Under the same threshold $u$, this envelope continues at known-arm reward $c$ whenever an original policy continues at some known-arm reward $v\ge c$.
Since $\bar\rho^{c,G}(h_n)\ge\rho^{c,G}(h_n)$, the envelope cannot stop earlier than the original policy.
Its possible cost is additional pulling when $p<c$.

To control this cost, fix a reward-shortfall tolerance $b>0$.
For each $c\in[0,1]$, set $\ell_b(h_0) \coloneqq 1$ and $\chi^{c,b}(h_0) \coloneqq 1$ and, for $1\le n<G$, define
\[
\ell_b(h_n) \coloneqq \min\left\{1,\min_{1\le j\le n}\frac{s_j+b}{j}\right\},
\qquad
\chi^{c,b}(h_n) \coloneqq \1\{c<\ell_b(h_n)\}.
\]
By its definition, the condition $\chi^{c,b}(h_n)=1$ means that $s_j>cj-b$ at every observation count $1\le j\le n$.
In other words, the cumulative reward of the unknown arm has never fallen at least $b$ below the reward of selecting the known arm for the same number of observations.
The running minimum over earlier counts are necessary because a stopping policy cannot resume pulling after it has stopped.
Define the adjusted cumulative continuation probability by
\begin{equation}
\label{eq:adjusted-continuation}
\rho^{\mathrm{adj},c,G,b}(h_n) \coloneqq \bar\rho^{c,G}(h_n)\chi^{c,b}(h_n).
\end{equation}
At the endpoints, $\rho^{\mathrm{adj},0,G,b}(h_n)=1$ and $\rho^{\mathrm{adj},1,G,b}(h_n)=0$ for every history.

\begin{proposition}
\label{prop:adjusted-regret}
For every integer $G\ge1$, $b>0$, $c\in[0,1]$ and a family of SAB stopping policies $\{\pi_{c, G}: c \in [0, 1]\}$, the probabilities in \eqref{eq:adjusted-continuation} define a new stopping policy $\pi^{\mathrm{adj}}_{c,G,b}$.
For each fixed history, they are Borel measurable and nonincreasing in the known-arm reward $c$.
They are also nonincreasing as observations are added.
For every $c\in[0,1]$,
\begin{equation*}
\begin{aligned}
R_{c,G}(\pi^{\mathrm{adj}}_{c,G,b},p)
\le R_{c,G}(\pi_{c,G},p) +G(p-c)e^{-8b(p-c)}, \quad p>c; \quad 
R_{c,G}(\pi^{\mathrm{adj}}_{c,G,b},p) <b+c,\quad  p<c.
\end{aligned}
\end{equation*}
\end{proposition}

The two adjustments have different purposes.
The monotone envelope gives the required ordering in the known-arm reward and does not increase regret when $p>c$.
The reward-shortfall policy bounds additional pulling when $p<c$, but can stop prematurely when $p>c$.
Its additional regret in that case is at most $G(p-c)e^{-8b(p-c)}$, which is important for the final MAB constant.

Using the same threshold at every observation count, let
$\tilde a^{\mathrm{adj},c,G,b}_{n+1} \coloneqq \mathbf1\{\tilde u<\rho^{\mathrm{adj},c,G,b}(\tilde h_n)\}$.
Combining Propositions~\ref{prop:shared-threshold} and~\ref{prop:adjusted-regret} gives
\begin{equation}\label{eq:continuation-bounds}
\begin{aligned}
\sum_{n=0}^{G-1}\mathbb P_p(\tilde a^{\mathrm{adj},c,G,b}_{n+1}=0)
\le\frac{R_{c,G}(\pi_{c,G},p)}{p-c}+Ge^{-8b(p-c)},\;\;p>c;
\quad 
\sum_{n=0}^{G-1}\mathbb P_p(\tilde a^{\mathrm{adj},c,G,b}_{n+1}=1)
\le\frac{b+c}{c-p},\;\;p<c.
\end{aligned}
\end{equation}
The first inequality will control index underestimation.
The second will control selections caused by index overestimation.

\subsection{The adjusted arm index}
\label{subsec:index}
Now, we are ready to define the index of arms for the MAB problem.
For a history $h_n$ with $0\le n<G$ and a threshold $ u\in[0,1)$, define
\begin{equation*}
\begin{aligned}
j_G(h_n, u)
 \coloneqq \sup\{c\in[0,1]:\rho^{c,G}(h_n)> u\},
\quad \text{and} \quad 
j^{\mathrm{adj}}_{G,b}(h_n, u)
 \coloneqq \sup\{c\in[0,1]:\rho^{\mathrm{adj},c,G,b}(h_n)> u\}.
\end{aligned}
\end{equation*}
Both sets contain $c=0$, since the original and adjusted cumulative continuation probabilities there equal one and $ u<1$.
Thus each index is well defined and equals zero if continuation is rejected at every positive known-arm reward.
At $c=1$, both cumulative continuation probabilities are zero, although the supremum may still equal one.
The adjusted index is the largest known-arm reward, in the supremum sense, at which the adjusted policy continues under threshold $ u$.
The next lemma states the formula used to evaluate it and the implications used in the regret proof.

\begin{lemma}
\label{lem:index}
For every $h_n$ with $0\le n<G$ and $ u\in[0,1)$,
\begin{equation}\label{eq:index-minimum}
j^{\mathrm{adj}}_{G,b}(h_n, u)
=\min\{j_G(h_n, u),\ell_b(h_n)\}.
\end{equation}
For every $c\in[0,1]$,
\[
\begin{aligned}
j^{\mathrm{adj}}_{G,b}(h_n, u)<c
\ \Longrightarrow\ \rho^{\mathrm{adj},c,G,b}(h_n)\le u,
\quad \text{and} \quad 
j^{\mathrm{adj}}_{G,b}(h_n, u)>c
\ \Longrightarrow\ \rho^{\mathrm{adj},c,G,b}(h_n)> u.
\end{aligned}
\]
The inequality $\rho^{\mathrm{adj},c,G,b}(h_n)\le u$ implies only
$j^{\mathrm{adj}}_{G,b}(h_n, u)\le c$.
The adjusted index is nonincreasing as observations are added and is Borel measurable in $ u$.
\end{lemma}

The monotone envelope leaves the original supremum index unchanged, while the reward-shortfall policy limits that index to $\ell_b(h_n)$.
Nevertheless, the envelope is needed to justify the continuation comparison when the index is strictly above $c$.
The strict inequalities also avoid assuming that a supremum is attained.
If $j^{\mathrm{adj}}_{G,b}(h_n, u)<c<p$, the auxiliary policy fails to continue with the better unknown arm.
If $j^{\mathrm{adj}}_{G,b}(h_n, u)>c>p$, it continues with the worse unknown arm.
The bounds in \eqref{eq:continuation-bounds} therefore connect the SAB regret analysis to the MAB index analysis.

%% file: main/section_3_mab.tex
\section{The Multi-Armed Policy and Its Regret Guarantee}
\label{sec:mab}

We now use the adjusted SAB indices to construct a SAB-based index  policy for the Bernoulli MAB problem.
With our proposed index evaluation, we prove that the policy has worst-case regret at most \(4.45\sqrt{KT}+10.75K\).
The \(\sqrt{KT}\) dependence is minimax-order optimal in view of the classical lower bound \(0.1365\sqrt{KT}\) of \citet[Theorem~6.11]{cesa2006prediction}.
The analysis connects the regret of the MAB policy to the regret
identities and bounds of the auxiliary SAB stopping policies.
Finally, we explain how to implement the policy by discretizing $p$ in the SAB linear programs and $c$ in the index search, and quantify the additional regret caused by these discretizations.

\subsection{Model and SAB-based index  policy}
\label{subsec:mab-model}

Consider $K \ge 2$ independent Bernoulli arms with fixed unknown mean vector $\mathbf p=(p_1,\ldots,p_K)\in[0,1]^K$ and integer horizon $T \ge 2$.
Arm $k$ has an independent reward sequence $\tilde z_{k,1},\tilde z_{k,2},\ldots\sim\operatorname{Ber}(p_k)$, and these sequences are mutually independent across arms.
At round $t$, an admissible policy selects $\tilde a_t\in[K] \coloneqq \{1,\ldots,K\}$ using past observations.
It observes the reward from the selected arm and no rewards from the other arms.
Let $p^\star \coloneqq \max_{k\in[K]}p_k$ and $\Delta_k \coloneqq p^\star-p_k$.

Let $\tilde n_{k,t} \coloneqq \sum_{i=1}^{t-1}\mathbf1\{\tilde a_i=k\}$ be the number of observations from arm $k$ available before round $t$.
Thus $\tilde n_{k,T+1}$ is its total number of pulls.
Write $\tilde h_{k,n} \coloneqq (\tilde z_{k,1},\ldots,\tilde z_{k,n})$ and $\tilde s_{k,n} \coloneqq \sum_{j=1}^n\tilde z_{k,j}$, with an empty history and zero successes at $n=0$.
Their realizations are denoted by $a_t$, $n_{k,t}$, $h_{k,n}$, and $s_{k,n}$.
The accumulated pseudo-regret is $\tilde r_T \coloneqq \sum_{t=1}^T\Delta_{\tilde a_t}$, and its expectation is
\begin{equation}
\label{eq:mab-regret}
R_T(\pi,\mathbf p) \coloneqq \mathbb E_{\mathbf p}[\tilde r_T]
=\sum_{k=1}^K\Delta_k\mathbb E_{\mathbf p}[\tilde n_{k,T+1}].
\end{equation}
We refer to this expected pseudo-regret simply as regret.

We now define the SAB-based index  policy.
Choose integers \(1\le H\le G\) and a tolerance \(b>0\).
For each \(c\in[0,1]\), fix an auxiliary SAB stopping policy
\(\pi_{c,G}\) with horizon \(G\) and known-arm reward \(c\), and let
\(j^{\mathrm{adj}}_{G,b}\) be the adjusted index constructed from this family.
Before the first MAB round, draw mutually independent thresholds
\(\tilde u_1,\ldots,\tilde u_K \sim \operatorname{Unif}(0,1)\).
The threshold \(\tilde u_k\) is attached to arm \(k\) and is used in every subsequent evaluation of that arm.
For each arm \(k\), define its random index before round \(t\) by
\begin{equation}
\label{eq:mab-index}
\tilde i_{k,t}(G,H,b)
\coloneqq
\begin{cases}
j^{\mathrm{adj}}_{G,b}
\bigl(\tilde h_{k,\tilde n_{k,t}},\tilde u_k\bigr),
& \tilde n_{k,t}<H,\\[0.8ex]
\displaystyle
\frac{\tilde s_{k,\tilde n_{k,t}}}{\tilde n_{k,t}},
& \tilde n_{k,t}\ge H.
\end{cases}
\end{equation}
Thus, the policy uses the adjusted SAB index until the arm has been observed \(H\) times and its empirical mean thereafter.

At each round, the policy selects an arm with the largest index, breaking ties in favor of the smallest arm number.
We denote the resulting policy by
\(\pi^{\mathrm{MAB}}_{G,H,b}\), suppressing its dependence on \(K\), \(T\), and the chosen SAB stopping policy family.
We note that the policy does not require a forced initialization.
At the empty history, \(\ell_b(h_0)=1\) and
\(\rho^{c,G}(h_0)=q^{c,G}_{0,0}\), so the initial index of arm \(k\) is
\(j^{\mathrm{adj}}_{G,b}(h_0,\tilde u_k)
=
\sup
\left\{
c\in[0,1]:
q^{c,G}_{0,0}>\tilde u_k
\right\}.\)
Thus, the initial continuation probabilities of the SAB stopping policies assign an index to every unobserved arm.
These probabilities determine only the initial indices; the MAB policy still selects an arm by comparing all \(K\) indices.

Algorithm~\ref{alg:max-index-policy} states the policy in terms of realized quantities.

\begin{algorithm}[htbp]
\caption{Maximum-index policy for the MAB problem}
\label{alg:max-index-policy}
\begin{algorithmic}[1]
\Require Integers \(K,T,G,H\); tolerance \(b>0\); a fixed SAB stopping policy family
\State Draw independent \(u_k\sim\operatorname{Unif}(0,1)\) for \(k\in[K]\).
\For{\(k=1,\ldots,K\)}
    \State Set \(n_{k,1}\gets 0\), \(h_{k,0}\gets\varnothing\), and \(s_{k,0}\gets 0\).
    \State Set \(i_{k,1}(G,H,b)\gets j^{\mathrm{adj}}_{G,b}
\bigl(h_{k,0}, u_k\bigr)\).
\EndFor
\For{\(t=1,\ldots,T\)}
    \State Select
    \(a_t\gets\min \, \arg\max_{k\in[K]}i_{k,t}(G,H,b)\).
    \State Set \(n\gets n_{a_t,t}\), pull arm \(a_t\), and observe its \((n+1)\)st reward \(z_{a_t,n+1}\).
    \State Set
    \(h_{a_t,n+1}\gets(h_{a_t,n},z_{a_t,n+1})\) and
    \(s_{a_t,n+1}\gets s_{a_t,n}+z_{a_t,n+1}\).
    \State Set
    \(n_{a_t,t+1}\gets n+1\) and update 
    \(i_{a_t,t+1}(G,H,b)\) according to~\eqref{eq:mab-index}.
    \State For each \(k\ne a_t\), retain
    \(n_{k,t+1}\gets n_{k,t}\) and
    \(i_{k,t+1}(G,H,b)\gets i_{k,t}(G,H,b)\).
\EndFor
\end{algorithmic}
\end{algorithm}
\ifdefined\FloatBarrier\FloatBarrier\fi

The parameters \(G\) and \(H\) have distinct roles.
The auxiliary horizon \(G\) is used to construct the SAB policies and remains fixed throughout the MAB horizon; it is not the number of MAB rounds remaining.
The switch to the empirical-mean index occurs separately for each arm when that arm reaches \(H\) observations.
Only the selected arm changes its history, observation count, and index in a given round, while every unselected arm retains its current index.
An arm with a low index is therefore not permanently eliminated and may be selected later if its index becomes maximal relative to the other arms.

\subsection{Algorithm performance}
\label{subsec:guarantee}

We now state the regret guarantee for the SAB-based index policy Algorithm~\ref{alg:max-index-policy}.
For the bound, we use
\begin{equation}
\label{eq:mab-parameters}
H \coloneqq \left\lceil\frac TK\right\rceil,
\quad
G \coloneqq 2H,
\quad
b \coloneqq \frac14\sqrt G .
\end{equation}
The theorem is stated for the SAB stopping policy family $\pi_{c,G}$ actually used by the policy.
For the result to hold, it is not necessary to require $\pi_{c,G}$ to be exact minimax policies; it is sufficient that the regret of $\pi_{c,G}$ is uniformly bounded for all $c \in [0, 1]$.

\begin{theorem}
\label{thm:mab-main}
Let \(T\ge 2K\), and choose \(H,G,b\) as in~\eqref{eq:mab-parameters}.
Let \(\{\pi_{c,G}:c\in[0,1]\}\) be the SAB stopping policy family used to construct the adjusted index \(j^{\mathrm{adj}}_{G,b}\).
Assume that this adjusted index is evaluated exactly for the chosen auxiliary family, and that the family satisfies the uniform SAB regret bound
\(\sup_{c,p\in[0,1]}
R_{c,G}(\pi_{c,G},p)
\le
\sqrt{\frac{G}{8e}}+\frac32 .\)
Then the SAB-based index policy satisfies
\begin{equation}
\label{eq:mab-main}
\sup_{\mathbf p\in[0,1]^K}
R_T(\pi^{\mathrm{MAB}}_{G,H,b},\mathbf p)
\le
4.45\sqrt{KT}+10.75K .
\end{equation}
For \(T<2K\), every admissible policy \(\pi\) satisfies
\(R_T(\pi,\mathbf p)\le T<\sqrt{2KT}\) for all \(\mathbf p\in[0,1]^K\).
\end{theorem}

The uniform regret assumption holds for the exact SAB minimax family by Proposition~\ref{cor:sab-upper} and for sufficiently accurate finite-grid constructions by Proposition~\ref{prop:sab-grid}.
The latter controls both the approximation error from discretizing the unknown mean $p$ and, through its Lipschitz bound in $c$, the additional regret from reusing policies computed at nearby known-arm rewards.
Section~\ref{subsec:implementation} describes the concrete implementation, including the construction of the finite policy bank and exact evaluation of the adjusted indices for the resulting auxiliary family.

The order of the bound in Theorem~\ref{thm:mab-main} cannot be improved in general.
For \(K \ge 2\), the Bernoulli instances used in the lower-bound construction of \citet[Theorem~6.11]{cesa2006prediction} imply
\(
\inf_\pi\sup_{\mathbf p\in[0,1]^K}R_T(\pi,\mathbf p)
\ge(\sqrt{2} -1) \sqrt{KT} / \sqrt{32 \log(4/3)} > 0.1365\sqrt{KT}.\)
Thus, the upper and lower bounds have the same dependence on \(K\) and \(T\), up to constants. 
Additionally, by the standard Bernoulli reduction, we can apply our algorithm to any MAB problem with bounded support as detailed in Remark~\ref{remark:extension}.

\begin{remark}
\label{remark:extension}
It is well known that any bounded-reward MAB problem can be reduced to a Bernoulli MAB problem; see, e.g., \citet{AgrawalGoyal2012,shipra2013further,xu2025bayesian,riou2020bandit}. For completeness, we briefly recall the reduction.

Without loss of generality, suppose that each arm \(k\) generates i.i.d.\ rewards with distribution \(\nu_k\) supported on \([0,1]\). Indeed, any bounded reward distribution \(\nu\) supported on \([a,b]\) can be mapped to this setting by the affine transformation $\widehat{\nu}=\frac{\nu-a}{b-a}.$ Assume further that the reward sequences are independent across arms. Let $p_k\coloneqq\int y\,\nu_k(dy).
$ Given a policy \(\pi\) for Bernoulli arms, define the corresponding policy \(\pi^{\mathrm{bin}}\) for bounded rewards as follows. After observing the \(n\)th reward \(\tilde y_{k,n}\) from arm \(k\), draw an independent random variable $\tilde v_{k,n}\sim\operatorname{Unif}(0,1)$ and feed only $\tilde z_{k,n}\coloneqq
\mathbf 1\{\tilde v_{k,n}<\tilde y_{k,n}\}$ to \(\pi\). All auxiliary uniform random variables are independent of the rewards and of the internal randomization of \(\pi\). Since \(\tilde z_{k,n}\sim\operatorname{Bernoulli}(p_k)\), this construction yields
$$
R_T(\pi^{\mathrm{bin}},(\nu_1,\ldots,\nu_K))
=
R_T(\pi, (p_1,\ldots,p_K)).
$$
Consequently, under the assumptions of Theorem~\ref{thm:mab-main}, the policy obtained by applying this reduction to \(\pi^{\mathrm{MAB}}_{G,H,b}\) satisfies~\eqref{eq:mab-main} uniformly over all such bounded-reward MAB problems.
\end{remark}

\begin{remark}
By choosing $H=\lceil\frac{32T}{K}\rceil$, $G=4H$ and $b=\sqrt H/20$, we can get a sharper theoretical
regret bound $1.51\sqrt{KT} + 1.9K$, while the leading constant $1.51$ is substantially below the best explicit coefficient $4$ established for Tsallis-INF policy in \citet[Theorem 1]{ZimmertSeldin2021}.
However, this improvement in the upper bound does not translate into better empirical performance: our numerical experiments
favor $G=2H$.
We therefore adopt $G=2H$ and $b=\sqrt G/4$ throughout, so that the theoretical guarantees apply to the same parameter setting used in the numerical experiments.
\end{remark}

\begin{remark}
\label{rem:fixed-threshold}
The random thresholds \(\tilde u_k\) are useful because they give the exact identity
\(\mathbb P(\tilde u<\rho)=\rho\).
This identity turns cumulative continuation probabilities into randomized continuation events and gives the continuation bounds used in the proof.
A deterministic version is also possible by replacing every arm-specific threshold with the fixed value \(1/2\). For this variant, the policy uses \(j^{\mathrm{adj}}_{G,b}(h_n,1/2)\) before an arm reaches \(H\) observations, and uses the empirical mean thereafter.
Thus the policy is deterministic once the rewards are realized.
The key observation is that, for every \(\rho\in[0,1]\),
\(\mathbf 1\{\rho\le1/2\}\le 2(1-\rho),\;
\mathbf 1\{\rho>1/2\}\le 2\rho .\)
Let \(
\tilde a^{\mathrm{det},c,G,b}_{n+1}
\coloneqq
\mathbf1\{\rho^{\mathrm{adj},c,G,b}(\tilde h_n)>1/2\}
\).
The two bounds in~\eqref{eq:continuation-bounds} will double when applies to $\tilde a^{\mathrm{det},c,G,b}_{n+1}$.
With the same choice \(H,\, G = 2H\) and the re-optimized tolerance
\(b=0.22\sqrt G\), 
and following the same optimization steps, we can show that the regret in this case is upper bounded by \(6.45\sqrt{KT}+17K\). 
\end{remark}

\subsection{Proof of Theorem~\ref{thm:mab-main}}
\label{subsec:proof-strategy}

We prove Theorem~\ref{thm:mab-main} by considering separately the scenarios why a suboptimal arm can have the largest index.
Throughout the proof, write
\(B_G
\coloneqq
\sup_{c,p\in[0,1]}R_{c,G}(\pi_{c,G},p),
\;
\kappa_G
\coloneqq
\frac{B_G}{\sqrt G}.\)
These quantities refer to the auxiliary SAB stopping policy family
actually used to construct the indices.
Under the hypothesis of Theorem~\ref{thm:mab-main},
\(B_G\le\sqrt{\frac{G}{8e}}+\frac32,
\;
\kappa_G\le\frac1{\sqrt{8e}}+\frac{3}{2\sqrt G}.\)

Suppose a suboptimal arm \(k\) is selected at round \(t\). There are two possibilities.
If its index is not much larger than its own mean, namely if
\(\tilde i_{k,t}(G,H,b)-p_k<\Delta_k/2\), then its index is below \(p^\star-\Delta_k/2\).
Because the selected arm has the largest index, every competing arm, including an optimal arm, must also have index below this level.
Thus the selection can be explained by underestimation of a better competing arm.
The other possibility is that the selected arm's own index is too high:
\(\tilde i_{k,t}(G,H,b)-p_k\ge\Delta_k/2\).
We call this event \emph{large index overestimation}.
We further split this large overestimation case according to whether the selected arm has fewer than or more than \(H\) observations, because the policy uses the adjusted SAB index in the first case and the empirical mean in the second. 

Let \(\tilde r_T^{\mathrm{no}}(G,H,b)\) be the regret from suboptimal selections that do not have large index overestimation.
For each suboptimal arm \(k\), let \(\tilde r_{k,T}^{<H}(G,H,b)\) be the regret from its large-overestimation selections before it reaches \(H\) observations, and define \(\tilde r_{k,T}^{\ge H}(G,H,b)\) similarly for selections after it reaches \(H\) observations.
Then
\[
\tilde r_T
=
\tilde r_T^{\mathrm{no}}(G,H,b)
+
\sum_{k:\Delta_k>0}\tilde r_{k,T}^{<H}(G,H,b)
+
\sum_{k:\Delta_k>0}\tilde r_{k,T}^{\ge H}(G,H,b).
\]
We bound these terms in order.

\subsubsection*{Selections without large index overestimation}
The main difficulty in bounding
\(\tilde r_T^{\mathrm{no}}(G,H,b)\)
is the complicated dynamics of arm indices and their correlation with the regret. 
We circumvent this complexity by introducing an arm-level envelope that is defined on the full underlying reward sequence of each arm and bound the contribution of regret $\tilde r_T^{\mathrm{no}}(G,H,b)$.
Here is the idea.

Imagine that, before the policy starts, each arm \(k\) is assigned an infinite independent reward sequence \(\tilde z_{k,1},\tilde z_{k,2},\ldots\)
and an independent threshold \(\tilde u_k\).
For each arm \(k\), given this reward sequence and threshold, define
\[
\tilde d_{k,G,H,b}
\coloneqq
\max\Big\{
\sup_{0\le n<H}
\Bigl(
p_k-j^{\mathrm{adj}}_{G,b}(\tilde h_{k,n},\tilde u_k)
\Bigr)_+,
\sup_{n\ge H}
\Big(
p_k-\frac{\tilde s_{k,n}}{n}
\Big)_+
\Big\}.
\]
Clearly, $\tilde d_{k,G,H,b}$ is an upper bound of the largest underestimation of arm $k$'s index \(\tilde i_{k,t}(G,H,b)\) along this realization of the arm's reward. 
It is an upper bound is because the realization of the rewards of arm $k$ in truncated at time $T$ while we are taking the maximization over the infinite sequence. 
By its definition, we also have, in every round \(t\),
\(\tilde i_{k,t}(G,H,b) \ge p_k-\tilde d_{k,G,H,b}\) at this realization of arm rewards.

We now use these envelopes to bound all selections contributing to
\(\tilde r_T^{\mathrm{no}}(G,H,b)\).
If every arm is optimal, this contribution is zero.
Otherwise, fix an optimal arm \(k^\star\) and a suboptimal arm \(k^\dagger\) with the smallest positive gap
\(\Delta_{\min} \coloneqq \min_{k:\Delta_k>0}\Delta_k = \Delta_{k^\dagger}.\)

Consider a round \(t\) in which a suboptimal arm \(k\) is selected and contributes to \(\tilde r_T^{\mathrm{no}}(G,H,b)\).
By definition of this case, the selected arm $k$ has no large overestimation, i.e., \(\tilde i_{k,t}(G,H,b)< p_k + \Delta_k/2 = p^\star-\frac{\Delta_k}{2}\) according to the definition of $\Delta_k$.
Since arm \(k\) is selected, its index is maximal.
Therefore \(\tilde i_{k^\star,t}(G,H,b) \le \tilde i_{k,t}(G,H,b).\)
Together with the envelope bound for the optimal arm, this gives
\[
p^\star-\tilde d_{k^\star,G,H,b}
\le
\tilde i_{k^\star,t}(G,H,b)
<
p^\star-\frac{\Delta_k}{2}, 
\quad \text{and hence} \quad
\tilde d_{k^\star,G,H,b}
>
\frac{\Delta_k}{2}.
\]
We also compare with the fixed closest suboptimal arm \(k^\dagger\).
If \(k\ne k^\dagger\), maximality of the selected index gives
\(\tilde i_{k^\dagger,t}(G,H,b)\le \tilde i_{k,t}(G,H,b)\).
If \(k=k^\dagger\), the same inequality holds with equality.
Thus, in all cases,
\[
p^\star-\Delta_{\min}-\tilde d_{k^\dagger,G,H,b}
\le
\tilde i_{k^\dagger,t}(G,H,b)
\le
\tilde i_{k,t}(G,H,b)
<
p^\star-\frac{\Delta_k}{2},
\quad  \text{and thus} \quad 
\Delta_{\min}+\tilde d_{k^\dagger,G,H,b}
>
\frac{\Delta_k}{2}.
\]

Define
\(\tilde v_{G,H,b}
\coloneqq
\min\Big\{
\tilde d_{k^\star,G,H,b},
\Delta_{\min}+\tilde d_{k^\dagger,G,H,b}
\Big\}.\)
The two comparisons above show that every selection counted by
\(\tilde r_T^{\mathrm{no}}(G,H,b)\)
satisfies \(\Delta_k<2\tilde v_{G,H,b}\).
Since every selected suboptimal arm has \(\Delta_k\ge\Delta_{\min}\), such a selection can occur only when \(\tilde v_{G,H,b}>\frac{\Delta_{\min}}{2}.\)
There are at most \(T\) selections in total, so we obtain the pathwise bound
\[
\tilde r_T^{\mathrm{no}}(G,H,b)
\le
2T\tilde v_{G,H,b}
\mathbf 1\left\{
\tilde v_{G,H,b}>
\frac{\Delta_{\min}}{2}
\right\}.
\]
For every nonnegative random variable \(\tilde v\) and every \(a>0\), according to integral by parts, \(
\mathbb E[\tilde v\mathbf 1\{\tilde v >a\}]
=
a\mathbb P[\tilde v>a]
+
\int_a^\infty \mathbb P[\tilde v>x]\,dx.\)
Applying this identity to
\(\tilde v_{G,H,b}\) and \(a=\Delta_{\min}/2\) gives
\[
\mathbb E_{\mathbf p}
\left[
\tilde r_T^{\mathrm{no}}(G,H,b)
\right]
\le
2T
\Big(
\frac{\Delta_{\min}}{2}
\mathbb P_{\mathbf p}\Big[\tilde v_{G,H,b}>\frac{\Delta_{\min}}{2}\Big]
+
\int_{\Delta_{\min}/2}^{\infty}
\mathbb P_{\mathbf p}
[\tilde v_{G,H,b}>x]\,dx
\Big).
\]
We now convert a pathwise bound into an expectation bound.
At this point, the adaptive MAB process no longer appears explicitly.
It is enough to control the tail of the random variable
\(\tilde v_{G,H,b}\).

Let \(\varphi:(0,\infty)\to[0,1]\) be a Borel measurable function such that
\(\mathbb P_{\mathbf p}
[\tilde d_{k,G,H,b}>x]
\le
\varphi(x),
\;
k\in[K],\ x>0.\)
We first derive a tail bound for \(\tilde v_{G,H,b}\).
For \(0<x<\Delta_{\min}\), the event
\(\tilde v_{G,H,b}>x\) is the same as
\(\tilde d_{k^\star,G,H,b}>x\), because the second term in the minimum satisfies
\(\Delta_{\min}+\tilde d_{k^\dagger,G,H,b}\ge \Delta_{\min}>x\).
Hence
\(\mathbb P_{\mathbf p}[\tilde v_{G,H,b}>x] \le \varphi(x),
\; 0<x<\Delta_{\min}.\)
For \(x>\Delta_{\min}\), both terms in the minimum must be larger than \(x\).
Thus \(\{\tilde v_{G,H,b}>x \}
\subseteq
\{\tilde d_{k^\star,G,H,b}>x\}
\cap
\{\tilde d_{k^\dagger,G,H,b}>x-\Delta_{\min}\}.\)
The two variables in this intersection are independent, because they are functions of different arms' full reward sequences and thresholds.
Therefore
\(\mathbb P_{\mathbf p}
[\tilde v_{G,H,b}>x]
\le
\varphi(x)\varphi(x-\Delta_{\min}), \; x>\Delta_{\min}\).
Using the one-arm bound on \([\Delta_{\min}/2,\Delta_{\min}]\) and the product bound on \([\Delta_{\min},\infty)\), we obtain
\[
\mathbb E_{\mathbf p}
\left[\tilde r_T^{\mathrm{no}}(G,H,b) \right]
\le
T \Big[\Delta_{\min}\varphi(\Delta_{\min}/2)
+
2\int_{\Delta_{\min}/2}^{\Delta_{\min}}\varphi(y)\,dy
+
2\int_{\Delta_{\min}}^\infty
\varphi(y)\varphi(y-\Delta_{\min})\,dy\Big].
\]
We set \(\varphi(0)=1\) only to make the endpoint \(y=\Delta_{\min}\) harmless in the last integral; changing a function at one point does not affect the integral.

It remains to construct a function \(\varphi\) satisfying \(\mathbb P_{\mathbf p}
[\tilde d_{k,G,H,b}>x]
\le
\varphi(x),
\;
k\in[K],\ x>0\). Define the generic single-arm envelope by
\[
\tilde d_{G,H,b}(p)
\coloneqq
\max\left\{
\sup_{0\le n<H}
\left(
p-j^{\mathrm{adj}}_{G,b}(\tilde h_n,\tilde u)
\right)_+,
\;
\sup_{n\ge H}
\left(
p-\frac{\tilde s_n}{n}
\right)_+
\right\}.
\]
For each arm $k$, the envelope $\tilde d_{k,G,H,b}$ has the same
distribution as $\tilde d_{G,H,b}(p_k)$.
We do this by bounding the two parts in the definition of
\(\tilde d_{G,H,b}(p)\).
The first part concerns the adjusted SAB index before \(H\) observations, and the second part concerns the empirical mean after \(H\) observations.

For the pre-\(H\) part, fix a Bernoulli sequence with mean \(p\), an independent threshold \(\tilde u\), and a number \(x>0\).
If \(x\ge p\), then the adjusted index cannot underestimate \(p\) by more than \(x\), because the index is nonnegative.
It is therefore enough to consider \(0<x<p\), and we set \(c=p-x\).
By Lemma~\ref{lem:index}, the adjusted index is nonincreasing as observations are added.
Thus the largest underestimation over \(0\le n<H\) occurs at \(n=H-1\), i.e., $\sup_{0\le n<H} (p_k-j^{\mathrm{adj}}_{G,b}(\tilde h_{k,n},\tilde u_k))_+ 
= (p_k-j^{\mathrm{adj}}_{G,b}(\tilde h_{k,H-1},\tilde u_k))_+$.
If \((p_k-j^{\mathrm{adj}}_{G,b}(\tilde h_{k,H-1},\tilde u_k))_+ > x\), then
\(j^{\mathrm{adj}}_{G,b}(\tilde h_{H-1},\tilde u)<p-x=c.\)
Lemma~\ref{lem:index} then implies that the adjusted SAB policy does not continue at known-arm reward \(c\) after history \(\tilde h_{H-1}\).
{This is the place where we need monotonicity of the index.}
Equivalently,
\(\tilde a^{\mathrm{adj},c,G,b}_H=0.\)
Since the adjusted SAB policy is a stopping policy, rejecting observation \(H\) also rejects observations \(H+1,\ldots,G\).
The first continuation bound in~\eqref{eq:continuation-bounds} gives, for \(1\le H\le G, b>0\), and  \(0\le c<p\le1\)
\begin{equation}\label{eq:bound_1}
\mathbb P_p [\tilde a^{\mathrm{adj},c,G,b}_H=0] \le
\frac{B_G}{(p-c)(G-H+1)} + \frac{Ge^{-8b(p-c)}}{G-H+1}.
\end{equation}
The factor \(G-H+1\) counts the number of periods when the SAB stopping policy does not pull the arms.
{ There is a place where the uniform regret bound of the SAB stopping policies is useful.}

For the post-\(H\) part, the MAB index is the empirical mean.
The maximal deviation bound in Lemma~\ref{lem:exponential} gives
\(\mathbb P_p
\left(
\sup_{n\ge H}
\left(
p-\frac{\tilde s_n}{n}
\right)
>x
\right)
\le
e^{-2Hx^2}.\)
Combining these pre and post-H part bounds yields the desired common tail bound.
\begin{equation}
\label{eq:underestimation-tail}
\mathbb P_p[\tilde d_{G,H,b}(p)>x ]
\le
\min\left\{
1,
\frac{B_G}{x(G-H+1)}
+
\frac{G}{G-H+1}e^{-8bx}
+
e^{-2Hx^2}
\right\}, \quad \forall p \in [0, 1],\, x >0.
\end{equation}
Therefore we may use
\[
\varphi(x)
\coloneqq
\min\left\{
1,
\frac{B_G}{x(G-H+1)}
+
\frac{G}{G-H+1}e^{-8bx}
+
e^{-2Hx^2}
\right\},
\quad x>0.
\]
The bound explain the reason that we take \(G>H\); if \(G=H\), $G-H+1$ factor would be only one and hence we might not be able to play around the probability.
This is also a place where we use the uniform regret bound of SAB stopping policy family to analyze the regret of the SAB-based index policy.

\subsubsection*{Selections with large index overestimation}

We next control the regret terms in which the selected suboptimal arm has an index that is too high.
Fix a suboptimal arm \(k\).
A selection of arm \(k\) has large index overestimation if \(\tilde i_{k,t}(G,H,b)-p_k\ge \frac{\Delta_k}{2}\).
We split this case further into two sub-cases: when the selected arm has been pulled fewer than $H$ observations or at least $H$ observations. 

We first consider the selections that occur before arm \(k\) reaches \(H\) observations.
At such a selection, the index of arm \(k\) is the adjusted SAB index
\(j^{\mathrm{adj}}_{G,b}(\tilde h_{k,n},\tilde u_k).\)
The large index overestimation assumption then implies
\(j^{\mathrm{adj}}_{G,b}(\tilde h_{k,n},\tilde u_k) \ge p_k+\frac{\Delta_k}{2}\).
For each fixed \(n\), this event is defined on the full reward sequence of arm \(k\), and a selection at count $n$ can happen at most once.
Therefore, by the definition of $\tilde r_{k,T}^{<H}(G,H,b)$, we have
\[
\mathbb E_{\mathbf p} \bigl[\tilde r_{k,T}^{<H}(G,H,b)\bigr]
\le
\Delta_k
\sum_{n=0}^{H-1}
\mathbb P_{p_k} \Big[j^{\mathrm{adj}}_{G,b}(\tilde h_{k,n},\tilde u_k)
\ge p_k+\frac{\Delta_k}{2} \Big].
\]

We now relate the event inside the sum to a wrong decision of the SAB stopping policy.
Choose
\(c_\varepsilon \coloneqq p_k+\frac{\Delta_k}{2}-\varepsilon,\; 0<\varepsilon<\frac{\Delta_k}{2}.\)
Then \(p_k<c_\varepsilon\).
Hence, for an auxiliary SAB problem with unknown-arm mean \(p_k\) and known-arm reward \(c_\varepsilon\), the known arm is better.
Continuing to sample the unknown arm is therefore the wrong decision.

Since \(\varepsilon > 0\), the large-overestimation event implies
\(j^{\mathrm{adj}}_{G,b}(\tilde h_{k,n},\tilde u_k)>c_\varepsilon.\)
By Lemma~\ref{lem:index}, this strict inequality implies
\(\rho^{\mathrm{adj},c_\varepsilon,G,b}(\tilde h_{k,n})>\tilde u_k\),
or equivalently, the adjusted SAB policy takes observation \(n+1\) at known-arm reward \(c_\varepsilon\).
Thus every large-overestimation event before \(H\) observations is contained in a wrong-continuation event for the auxiliary SAB problem.
Since \(H\le G\), using the second continuation bound in~\eqref{eq:continuation-bounds}, with \(p=p_k\) and \(c=c_\varepsilon\), gives
\(\mathbb E_{\mathbf p}\bigl[\tilde r_{k,T}^{<H}(G,H,b)\bigr]
\le
\Delta_k \sum_{n=0}^{G-1} \mathbb P_{p_k}[\tilde a^{\mathrm{adj},c_\varepsilon,G,b}_{n+1}=1] 
\le \Delta_k \frac{b+c_\varepsilon}{c_\varepsilon-p_k}\).
Letting \(\varepsilon\downarrow0\) gives
\begin{equation} \label{eq:large_overestimation_short}
\mathbb E_{\mathbf p} \bigl[ \tilde r_{k,T}^{<H}(G,H,b) \bigr]
\le
\Delta_k
\frac{b+p_k+\Delta_k/2}{\Delta_k/2}
=
2\left(b+p_k+\frac{\Delta_k}{2}\right)
\le
2b+2,
\end{equation}
because \(p_k+\Delta_k=p^\star\le1\).

We now consider selections after arm \(k\) has reached \(H\) observations.
In this range the policy uses the empirical mean as the index.
Therefore a large-overestimation selection at count \(n\ge H\) requires
\(\frac{\tilde s_{k,n}}{n}-p_k\ge\frac{\Delta_k}{2}.\)
Then, the expected regret under this can can be bounded by 
\(\mathbb E_{\mathbf p}
\bigl[\tilde r_{k,T}^{\ge H}(G,H,b) \bigr] \le
\Delta_k \sum_{n=H}^{\infty} \mathbb P_{p_k}\big[\frac{\tilde s_{k,n}}{n}-p_k\ge\frac{\Delta_k}{2}\big].\)
By Lemma~\ref{lem:exponential}, the probability in the \(n\)th term is at most
\(e^{-n\Delta_k^2/2}\).
Thus
\[
\mathbb E_{\mathbf p}\bigl[\tilde r_{k,T}^{\ge H}(G,H,b)\bigr]
\le 
\Delta_k\sum_{n=H}^{\infty}e^{-n\Delta_k^2/2}
\le
e^{-H\Delta_k^2/2}
\Big(\Delta_k+\frac{2}{\Delta_k}
\Big).
\]

Summing the pre-\(H\) bound over the at most \(K-1\) suboptimal arms gives
\((K-1)(2b+2)\).
However, the post-\(H\) term needs one additional step.
The regret bound contains the factor \(2/\Delta_k\), which can be large for very small $\Delta_k$.
However, small-gap arms create little regret per pull, and there are at most \(T\) pulls in total.
To use this fact, fix an analysis threshold \(\delta>0\) and decompose the regret of each post-\(H\) large-overestimation selection as
\(\Delta_k
=
\min\{\Delta_k,\delta\}
+
(\Delta_k-\delta)_+.\)
The first parts contribute at most \(\delta T\) in total over all rounds.
For an arm with \(\Delta_k>\delta\), the remaining part is the fraction
\(1-\frac{\delta}{\Delta_k}\)
of the original regret of each such selection.
Therefore, we can bound the post-\(H\) large index overestimation bound as
\[
\mathbb E_{\mathbf p} \Big[\sum_{k:\Delta_k>0} \tilde r_{k,T}^{\ge H}(G,H,b) \Big]
\le 
\inf_{\delta > 0}\Big\{\delta T + (K-1) \sup_{\Delta>\delta}
\Big(1-\frac{\delta}{\Delta} \Big)
e^{-H\Delta^2/2}
\Big(\Delta+\frac{2}{\Delta}
\Big)\Big\}.
\]

\subsubsection*{Combining and minimizing the total regret.}
We now put the pieces together.
Taking the worst-case $\Delta_{\min} > 0$ and best value of \(\delta\) gives the following general bound.
\begin{equation}
\label{eq:master}
\begin{aligned}
R_T(\pi^{\mathrm{MAB}}_{G,H,b},\mathbf p)
\le{}&
T\sup_{\Delta_{\min}>0}\Big[\Delta_{\min}\varphi(\Delta_{\min}/2)
+
2\int_{\Delta_{\min}/2}^{\Delta_{\min}}\varphi(y)\,dy
+
2\int_{\Delta_{\min}}^\infty
\varphi(y)\varphi(y-\Delta_{\min})\,dy\Big] 
\\ &
+(K-1)(2b+2) + 
\inf_{\delta > 0}\Big\{\delta T + (K-1) \sup_{\Delta>\delta}
\Big(1-\frac{\delta}{\Delta} \Big)
e^{-H\Delta^2/2}
\Big(\Delta+\frac{2}{\Delta}
\Big)\Big\}.
\end{aligned}
\end{equation}
The supremum over \(\Delta>\delta\) is taken over a slightly larger range than the possible Bernoulli gaps, which only makes the bound larger and simplifies the expression.

Increasing \(b\) lowers the probability of premature stopping but increases the pre-switch overestimation term \((K-1)(2b+2)\).
When \(B_G\) is of order \(\sqrt G\), choosing \(G/H\) as a fixed constant larger than one and \(b\) proportional to \(\sqrt H\) makes the underestimation contribution of order \(T/\sqrt H\) and the overestimation contributions of order \(K\sqrt H\).
Balancing these two orders leads to \(H\) of order \(T/K\).
This is the reason for the choice \(H=\lceil T/K\rceil\), \(G=2H\), and \(b=\sqrt G / 4\).

For \(\kappa>0\), set \(\Phi_\kappa(y)=1\) for \(y\le0\), and define
\begin{equation}
\label{eq:phi}
\Phi_\kappa(y)
\coloneqq
\min\left\{1,\frac{\sqrt2\,\kappa}{y}
+2e^{-2\sqrt2\,y}+e^{-2y^2}\right\},\quad y>0.
\end{equation}
For a Borel measurable function \(\varphi:(0,\infty)\to[0,1]\), define
\begin{equation}
\label{eq:tail-integral}
\begin{aligned}
J_\varphi(d)
&\coloneqq d\varphi(d/2)+2\int_{d/2}^{d}\varphi(y)\,dy
+2\int_d^\infty\varphi(y)\varphi(y-d)\,dy, \quad \text{and} \quad 
\Lambda_0(\kappa) \coloneqq\sup_{d>0}J_{\Phi_\kappa}(d).
\end{aligned}
\end{equation}
The endpoint value of \(\varphi\) in the last integral can be chosen arbitrarily.
Since \(B_G=\kappa_G\sqrt G=\sqrt2\,\kappa_G\sqrt H\),
\(G-H+1=H+1\), and \(8b=2\sqrt{2H}\),
\eqref{eq:underestimation-tail} is bounded by
\(\Phi_{\kappa_G}(x\sqrt H)\).
In particular, we may use \(\varphi(x)=\Phi_{\kappa_G}(x\sqrt H)\) in \eqref{eq:master}.
A change of variables gives
\[
\sup_{d>0}J_\varphi(d)=\frac{\Lambda_0(\kappa_G)}{\sqrt H},
\quad
\mathbb E_{\mathbf p}[\tilde r_T^{\mathrm{no}}(G,H,b)]
\le\frac{T\Lambda_0(\kappa_G)}{\sqrt H}
\le\Lambda_0(\kappa_G)\sqrt{KT}.
\]

It remains to bound the post-switch contribution.
Let $\gamma_\star>0$ be the unique solution of
$\gamma_\star^2e^{\gamma_\star^2/2}=2$, and set \(\gamma_0\coloneqq\gamma_\star(\gamma_\star^2+1)/(\gamma_\star^2+2)\). 
Taking $\delta=\gamma_0/\sqrt H$ in~\eqref{eq:master},
applying~\eqref{eq:post-switch-scalar-bound}, and using
$2b=\sqrt H/\sqrt2$, we obtain
\[
\begin{aligned}
R_T(\pi^{\mathrm{MAB}}_{G,H,b},\mathbf p)
\le{}
\frac{T\Lambda_0(\kappa_G)}{\sqrt H}
+2(K-1)
+\gamma_0\frac{T}{\sqrt H} +
\left(
\frac1{\sqrt2}+\gamma_\star-\gamma_0
\right)(K-1)\sqrt H
+\frac{K-1}{\sqrt{eH}}.
\end{aligned}
\]
For \(T\ge2K\), rounding \(T/K\) up to \(H\) gives
\(\frac T{\sqrt H}\le\sqrt{KT},\; K\sqrt H\le\sqrt{KT}+\frac K{2\sqrt2},\) and  \(\frac1{\sqrt{eH}}\le\frac1{\sqrt{2e}}.\)
Since
\(2+\frac{1/\sqrt2+\gamma_\star-\gamma_0}{2\sqrt2}
+\frac1{\sqrt{2e}}
<2.799020<\frac{14}{5}\),
we conclude that
\begin{equation*}
R_T(\pi^{\mathrm{MAB}}_{G,H,b},\mathbf p)
\le\left[\Lambda_0(\kappa_G)+\frac1{\sqrt2}+\gamma_\star\right]\sqrt{KT}
+\frac{14}{5}K.
\end{equation*}

Finally, put \(\kappa_0=1/\sqrt{8e}\).
The hypothesis of Theorem~\ref{thm:mab-main} gives
\(\kappa_G\le\kappa_0+3/(2\sqrt G)\).
The numerical bounds in Lemma~\ref{lem:quantity-kappa-G} are
\(\Lambda_0(\kappa_0)<2.671,
\;
0<\Lambda_0'(\kappa)\le\sqrt2(2+3\log3),\)
and hence
\(\Lambda_0(\kappa_G)
\le2.671+\frac{3\sqrt2}{2\sqrt G}(2+3\log3).\)
Because \(G=2H\ge2T/K\), substitution into the above regret gives
\[
\begin{aligned}
R_T(\pi^{\mathrm{MAB}}_{G,H,b},\mathbf p)
\le{}&\Big(2.671+\frac1{\sqrt2}+\gamma_\star\Big)\sqrt{KT}+\left[\frac{14}{5}+\frac32(2+3\log3)\right]K\;
< \;4.45\sqrt{KT}+10.75K.
\end{aligned}
\]
This proves \eqref{eq:mab-main}. The scalar root and integral calculations are certified in Appendix~\ref{subsec:lambda-evaluation}.

\subsection{Implementation and numerical approximation}
\label{subsec:implementation}

Theorem~\ref{thm:mab-main} requires a uniform SAB regret bound and exact index evaluation from the SAB stopping policy family actually used.
Both requirements can be met using finitely many SAB programs.
We discretize the unknown mean $p$ within each program and use a finite grid of known-arm rewards $c$ to define a piecewise-constant SAB stopping policy family.
The index of this family is then evaluated exactly, rather than approximating the index of the minimax family over the continuum $c\in [0, 1]$.

Fix the auxiliary horizon $G$, an integer $M\ge1$, and a grid $\mathcal C=\{v_0,\ldots,v_L\}, \, 0=v_0<\cdots<v_L=1$.
At each interior point $v_i$, solve the SAB program with the unknown-mean grid $\mathcal P_M$ from Section~\ref{subsec:sab-bounds} and recover the table $q^{v_i,G,M}_{n,s}$ using~\eqref{eq:sab-recovery}.
At $v_0=0$ and $v_L=1$, use the analytical policies that always continue and always stop, respectively.
For each $c\in[0,1]$, let $v(c)$ be a nearest grid point, choosing the smaller one in a tie, and define
\(\pi_{c,G}\coloneqq\pi_{v(c),G,M}, \;
q^{c,G}_{n,s}\coloneqq q^{v(c),G,M}_{n,s}.\)
Thus the same stopping policy is used at every known-arm reward assigned to a given grid point.

Write $\varepsilon_p\coloneqq\frac{G(G+1)}{4M}$ and $\varepsilon_c\coloneqq\frac12\max_{0\le i<L}(v_{i+1}-v_i).$ Proposition~\ref{prop:sab-grid} directly gives, for every $c,p\in[0,1]$,
\[
\begin{aligned}
R_{c,G}(\pi_{c,G},p)
&\le R_{v(c),G}(\pi_{v(c),G,M},p)+G|c-v(c)|\le V_{G,M}(v(c))+\varepsilon_p+G\varepsilon_c.
\end{aligned}
\]
Since $V_{G,M}(v_i)\le V_G(v_i)$, Proposition~\ref{cor:sab-upper} yields
\begin{equation*}
\sup_{c,p\in[0,1]}R_{c,G}(\pi_{c,G},p)
\le \max_{0\le i\le L}V_{G,M}(v_i)+\varepsilon_p+G\varepsilon_c\le \sqrt{\frac{G}{8e}}+1+\varepsilon_p+G\varepsilon_c.
\end{equation*}
The first bound also provides a certificate from the finite-grid optimal values.
Consequently, if $\varepsilon_p+G\varepsilon_c\le\frac12,$
then this auxiliary family satisfies the regret assumption of Theorem~\ref{thm:mab-main}.
For the uniform grid $v_i=i/L$, the sufficient choices $L\ge2G$ and $M\ge G(G+1)$ make each error term at most $1/4$.
With $H,G,b$ chosen as in~\eqref{eq:mab-parameters}, the resulting MAB policy therefore retains the bound~\eqref{eq:mab-main}.
These are sufficient grid sizes for the stated guarantee, not requirements on every implementation.

For a history $h_n=(z_1,\ldots,z_n)$, let $s_j=\sum_{r=1}^j z_r$ and compute $\rho^{v_i,G,M}(h_n)
\coloneqq\prod_{j=0}^n q^{v_i,G,M}_{j,s_j}.$
Define
\[
\xi_i\coloneqq\frac{v_i+v_{i+1}}2\quad(0\le i<L),
\quad \xi_L\coloneqq1.
\]
The set of rewards assigned to $v_i$ is an interval with supremum $\xi_i$.
For a given threshold $u\in[0,1)$, the accepted set is the union of those intervals whose cumulative continuation probability exceeds $u$.
Its supremum is therefore
\begin{equation}
\label{eq:finite-family-index}
\begin{aligned}
j^{\mathrm{aux}}_G(h_n,u)
&=\max\{\xi_i:0\le i\le L,\ \rho^{v_i,G,M}(h_n)>u\},\\
j^{\mathrm{adj}}_{G,b}(h_n,u)
&=\min\{j^{\mathrm{aux}}_G(h_n,u),\ell_b(h_n)\},
\end{aligned}
\end{equation}
where the second equality follows from Lemma~\ref{lem:index}.
The maximum is nonempty because the policy at $v_0=0$ always continues.
The quantities maximized are the interval endpoints $\xi_i$, not the representative grid points $v_i$.
The cumulative probabilities need not be ordered in $i$, so they can be scanned to evaluate the maximum without any monotonicity assumption in the known-arm reward.
No explicit computation of the monotone envelope is needed.

For each arm, initialize the cumulative probabilities at $q^{v_i,G,M}_{0,0}$ and the reward-shortfall bound at $\ell_b(h_0)=1$.
After that arm receives its $n$th observation, update, for $1\le n<H$,
\[
\rho^{v_i,G,M}(h_n)
=\rho^{v_i,G,M}(h_{n-1})q^{v_i,G,M}_{n,s_n},
\quad
\ell_b(h_n)=\min\left\{\ell_b(h_{n-1}),\frac{s_n+b}{n}\right\},
\]
and evaluate~\eqref{eq:finite-family-index} with the arm's threshold $u$.
Only the selected arm requires an update.
Once its count reaches $H$, use its empirical mean as prescribed by~\eqref{eq:mab-index}.
Thus an index update costs $O(L)$ before the switch and $O(1)$ thereafter; scanning the $K$ current indices to choose an arm costs $O(K)$ per round.
Only the table entries with $n<H$ are needed online, so shared continuation tables and per-arm cumulative probabilities require $O(LH+KL)$ storage, without retaining complete ordered histories.

%% file: main/numerical_study.tex
\section{Numerical Experiments}
\label{sec:experiments}

This section examines the numerical performance of the SAB minimax policy and the corresponding SAB-based index policy for the MAB problem.
In the single-armed experiments, the SAB policy has substantially smaller worst-case regret than the benchmark policies, with regret roughly half that of the best benchmark at the tested horizons up to \(T=200\).
We then examine the SAB continuation decisions to understand this difference: at a fixed empirical mean, the SAB policy is less willing than benchmark policies to continue sampling at high known-arm rewards, particularly at small pull counts.
In the multi-armed experiments, the SAB-based index policy has lower regret than every benchmark policy across all tested \((K,T)\) combinations.
For \(K=2\), its regret is close to that of the grid-based MAB minimax policy obtained by discretizing the semi-infinite LP~\eqref{eq:mab-lp}, which only adds on small regret to the minimax optimal policy.

We compare the SAB minimax policy and the SAB-based index policy, labeled SAB-LP-MAB in the figures, with five benchmark policies that are either standard in the literature or closely related to minimax exploration: KL-UCB \citep{GarivierCappe2011}, KL-UCB-Switch \citep{GarivierHadijiMenardStoltz2022}, MOSS \citep{AudibertBubeck2009}, KL-MOSS \citep{GarivierCappe2011,AudibertBubeck2009}, $\varepsilon$-Thompson sampling ($\varepsilon$-TS) \citep{JinYangXiaoXu2023}, and Tsallis-INF \citep{ZimmertSeldin2021}.
Let \(n_k\) and \(\widehat\mu_k\) denote the pull count and empirical mean of arm \(k\) before round \(t\).
Except for \(\varepsilon\)-TS, the benchmark policies first pull every arm once; exact index ties are broken uniformly at random.
We briefly describe these policies below.

The {KL-UCB} policy replaces the classical UCB quadratic confidence radius with a Bernoulli KL constraint:
\begin{equation*}
  I_{k,t}^{\mathrm{KL-UCB}}
  =\sup\left\{q\in[\widehat\mu_k,1]:
      n_k\text{KL}(\widehat\mu_k,q)\le \log t\right\},
\end{equation*}
where
\(\text{KL}(p,q)=p\log(p/q)+(1-p)\log((1-p)/(1-q))\).
The MOSS policy uses the known-horizon index
\begin{equation*}
  I_{k,t}^{\mathrm{MOSS}}
  =\widehat\mu_k+
   \sqrt{\frac{\logplus(T/(K n_k))}{2 n_k}},
  \qquad
  \logplus(x)=\max\{0,\log x\}.
\end{equation*}
It is a modification of the classical distribution-free minimax policy of~\cite{AudibertBubeck2009} to reduce the regret \citep{GarivierHadijiMenardStoltz2022}.
The {KL-MOSS} policy is the KL analogue of MOSS used in this study:
\begin{equation*}
  I_{k,t}^{\mathrm{KL-MOSS}}
  =\sup\left\{q\in[\widehat\mu_k,1]:
      n_k\kl(\widehat\mu_k,q)
      \le \logplus\!\left(\frac{T}{K n_k}\right)\right\}.
\end{equation*}
Thus, KL-MOSS combines the Bernoulli KL confidence set of~\cite{GarivierCappe2011} with the MOSS exploration budget of~\cite{AudibertBubeck2009}.
The {KL-UCB-Switch} policy follows~\cite{GarivierHadijiMenardStoltz2022}.
In our implementation, it uses the KL-MOSS index while
\(n_k\le\lfloor(T/K)^{1/5}\rfloor\), and then switches to the modified MOSS index $I_{k,t}^{\mathrm{MOSS}}$. 
The {\(\varepsilon\)-TS} benchmark is the reduced-exploration Thompson sampling policy of~\cite{JinYangXiaoXu2023}.
We use \(\varepsilon=1/2\) and a \(\operatorname{Beta}(1,1)\) prior.
For each arm at each decision, we sample its score from \(\operatorname{Beta}(S_k+1,n_k-S_k+1)\) with probability \(1/2\), and otherwise use its empirical mean.
This policy has no forced initialization; the empirical score of an unpulled arm is defined as zero.

We implement the SAB-based index policy using the discretization procedure in Section~\ref{subsec:implementation}, with $M$ and $L$ grid points for $p$ and $c$, respectively, on $[0,1]$.
The choices of $M$ and $L$ depend on the horizon $T$ in each experiment.
For each interior point $c$ on the $L$-point grid, we precompute the SAB stopping policy using a uniform $M$-point optimization grid for $p \in [0,1]$.
Only the continuation boundary of each policy needs to be stored, requiring $\mathcal{O}(LT)$ memory.
The remaining steps follow Section~\ref{subsec:implementation}, with the corresponding parameter choices.
These precomputed SAB stopping policies are reused across simulations, so the SAB linear programs do not need to be solved again.

Since our objective is to reduce finite-horizon worst-case regret, we compare all aforementioned algorithms with our proposed method in terms of their worst-case regret over the considered problem class.  For the Monte Carlo experiments, we approximate worst-case regret using a coarse-to-fine search rather than an exhaustive grid. For each algorithm and horizon, we construct a broad candidate set comprising structured instances, boundary cases, and low-discrepancy points. We initially screen each candidate using 256 simulations, reevaluate 24 promising candidates using 2,048 simulations each, and retain 8 for four rounds of local refinement with perturbation scales \(0.08,0.04,0.02,0.01\), applying the same screening and reevaluation procedure in each round. The resulting finalists, supplemented with previously identified difficult instances, are compared using 20,000 simulations per instance. Finally, we fix the selected instance and report its mean regret from 100,000 independent simulations, avoiding reuse of the simulation noise involved in selection. This procedure provides an empirical estimate, rather than a global certificate, of worst-case regret.

\subsection{Single-armed bandit experiments}
\label{subsec:sab-experiments}

For each benchmark policy, we compute the unknown arm's index or score using the expressions above, with $K=2$ where applicable.
We use $K=2$ rather than $K=1$ because it gives the benchmark policies smaller regret in our experiments.
Since the worst-case regret over the continuous parameter space is difficult to evaluate directly, we discretize $c\in[0,1]$ into $L= 100$ grid points and $p\in[0,1]$ uniformly into $M =1,000$ grid points.
For each cost grid point \(c\), we evaluate every benchmark policy at all \(M\) grid points for \(p\). Its expected regret at each \((c,p)\) is computed by probability-based exact dynamic programming rather than Monte Carlo simulation, and we report the maximum over the discretized \((c,p)\) domain.
For the same $c$, we solve the finite-grid discretization of LP~\eqref{eq:sab-lp} on the same $p$-grid to obtain the grid-based SAB minimax policy, labeled SAB-LP in the figures.
The worst-case regret of this policy over the $p$-grid equals the optimal objective value $\eta$.
For each policy and horizon, we then take the largest of these values over the $L$ grid points for $c$.
Figure~\ref{fig:exp-sab-regret} reports the resulting worst-case regrets over the $(c,p)$-grid across horizons, with the benchmark regrets computed by exact probability-based dynamic programming and the SAB-LP regret obtained from the LP objective values.

At all ten horizons, the SAB policy has substantially lower regret than every benchmark policy, with worst-case regret roughly half that of the best benchmark.
Figure~\ref{fig:exp-sab-states} shows a representative SAB policy at \(T=20\) and \(c=0.5\).
Filled circles denote pulling the unknown arm, hollow circles denote stopping and pulling the known arm, and half-filled circles denote randomized actions.
For each fixed pull count \(n\), the continuation set is an upper set in the success count: sufficiently many successes lead to continuation.
The policy may randomize at certain states; for example, at state $(1,0)$, it pulls the unknown arm with probability $0.6077$.
For a given $c$, we observe very few randomized states, even at $T=200$.

\begin{figure}[htbp]
  \centering
  \begin{subfigure}[t]{0.485\linewidth}
    \centering
    \includegraphics[width=\linewidth]{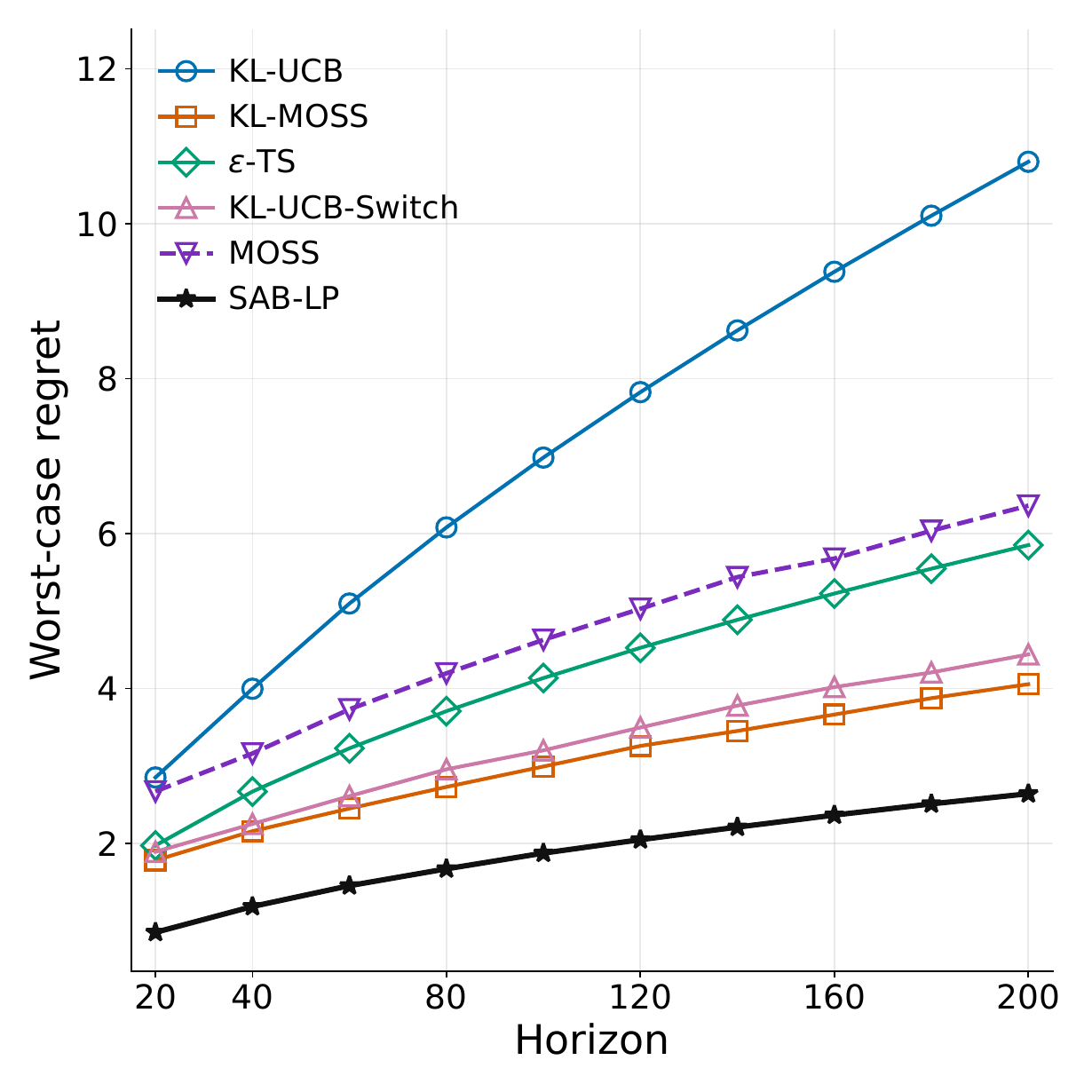}
    \caption{Worst-case regret versus horizon.}
    \label{fig:exp-sab-regret}
  \end{subfigure}\hfill
  \begin{subfigure}[t]{0.485\linewidth}
    \centering
    \includegraphics[width=\linewidth]{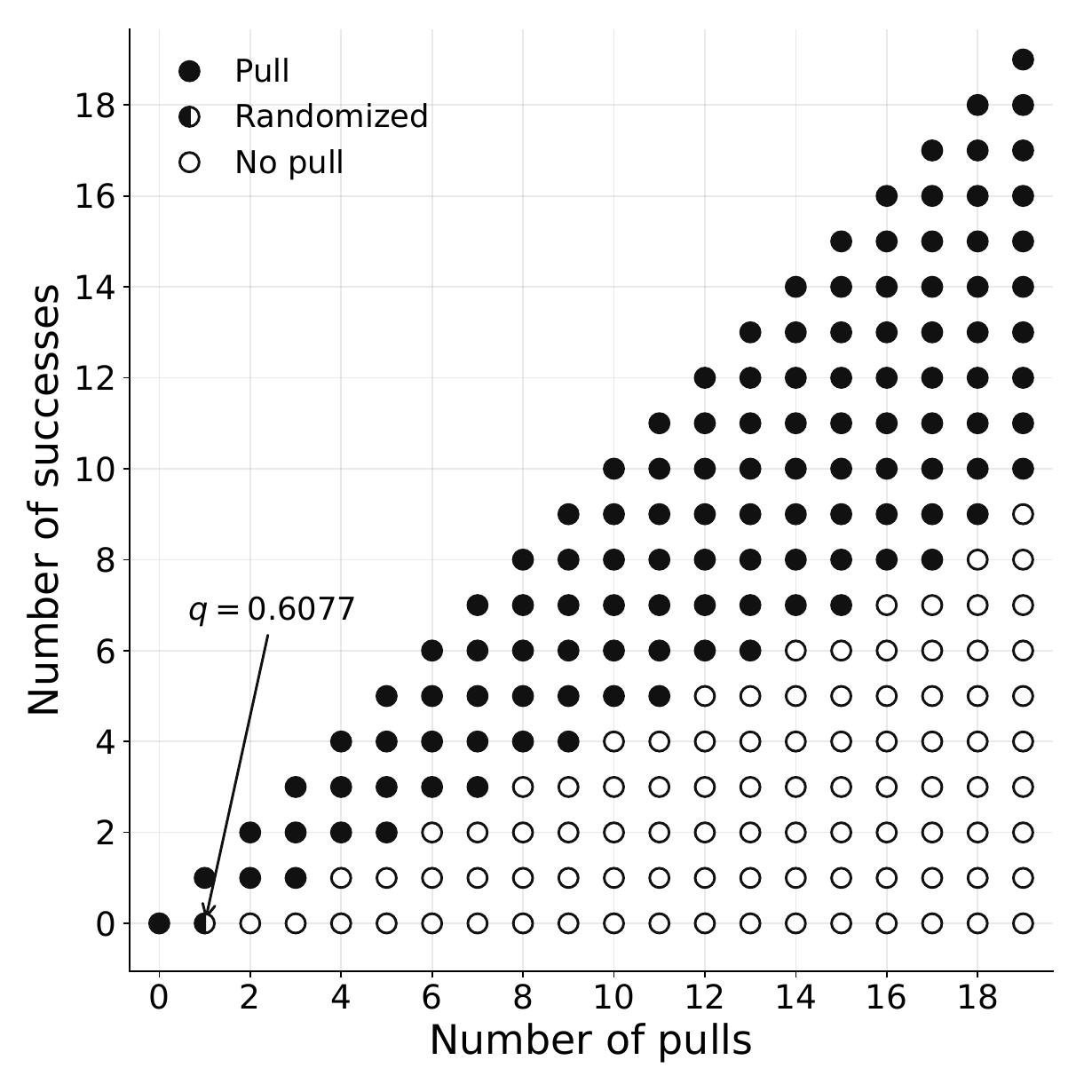}
    \caption{State decisions at \(T=20,c=0.5\).}
    \label{fig:exp-sab-states}
  \end{subfigure}
  \caption{Single-arm performance and continuation decisions. In the right panel, filled black circles denote pulling the unknown arm, hollow circles denote stopping, and the half-filled circle denotes the randomized state of the SAB policy.}
  \label{fig:exp-sab-pair}
\end{figure}

Figure~\ref{fig:exp-sab-boundary} helps explain the regret difference.
We fix \(T=200\), consider even pull counts \(n\) within the horizon, and set the success count to \(s=n/2\).
The figure shows the known-arm reward boundary for pulling the unknown arm with probability greater than \(1/2\) at state $(n,n/2)$.
For the SAB policy, we locate the crossing \(q_c(n,n/2;200)=1/2\) by local linear interpolation between adjacent points on the \(0.01\)-spaced $c$-grid whose continuation probabilities contains \(1/2\).
For the benchmark policies with deterministic indices, we obtain the boundary directly from their analytic index formulas.
For this statewise comparison, we evaluate KL-UCB at the next decision, so its exploration budget is \(\log(n+1)\).

\begin{figure}[htbp]
  \centering
  \includegraphics[width=0.78\linewidth]{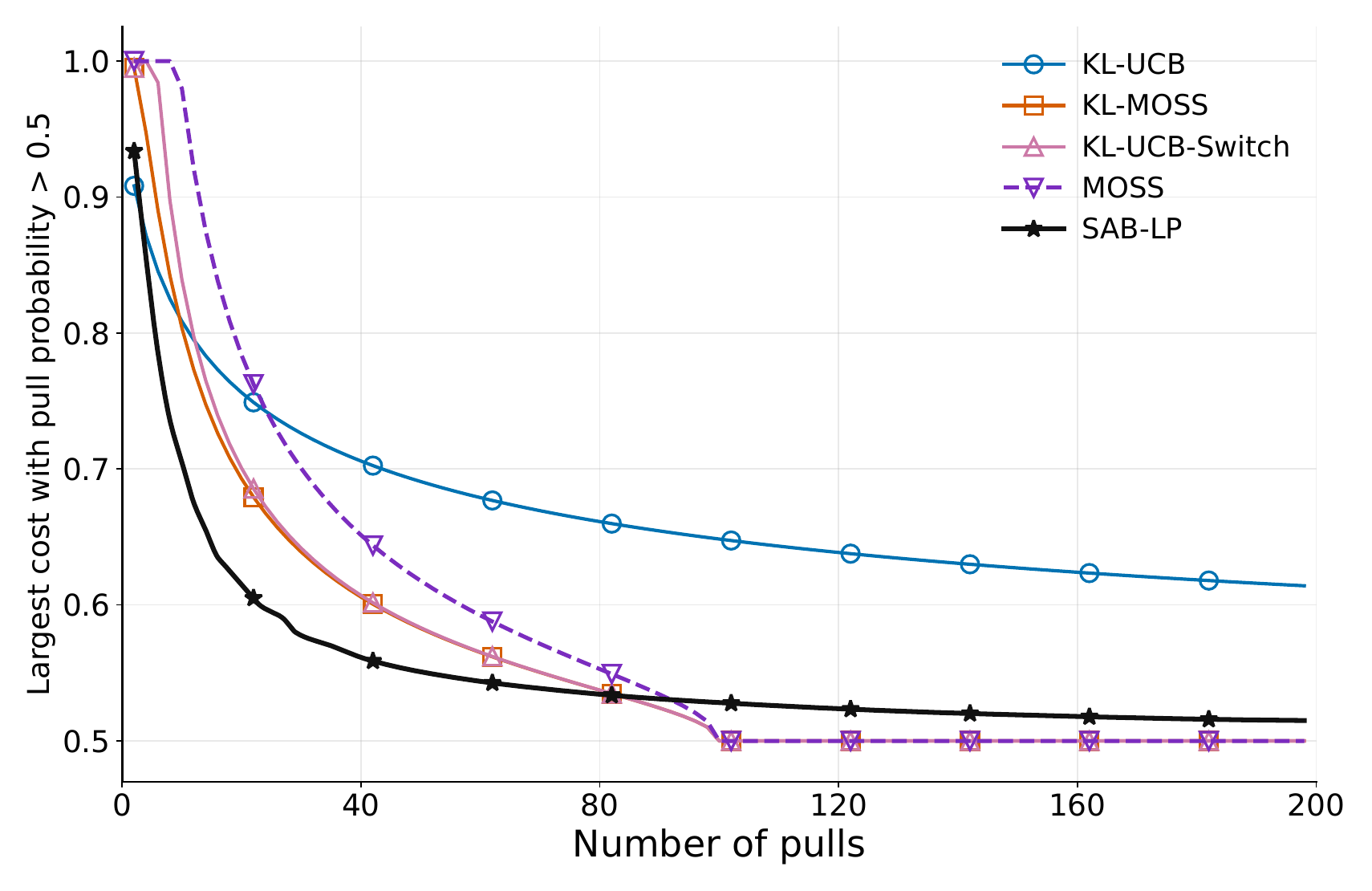}
  \caption{Known-arm reward boundary for pulling with probability greater than \(1/2\), at
  \(T=200\) and \(s=n/2\). The horizontal line is the empirical mean
  \(s/n=0.5\).}
  \label{fig:exp-sab-boundary}
\end{figure}

When the pull count is below about $80$, the benchmark policies retain a larger continuation margin above the empirical mean \(s/n=0.5\): they continue sampling at known-arm rewards for which the SAB policy already favors stopping.
Beyond this threshold, the SAB continuation boundary soon exceeds those of the benchmark policies, while remaining only slightly above the empirical mean.
As a substantial part of the horizon remains at this point, so limited further exploration can reduce the loss from abandoning a promising unknown arm too early.
Thus, the SAB policy is more conservative at small pull counts but allows a small exploration margin at larger pull counts while there is still a cost to missing a promising arm.
This pattern suggests an explanation for the approximately twofold regret difference in Figure~\ref{fig:exp-sab-regret}.
The \(\varepsilon\)-TS boundary equals \(0.5\) under this symmetric probability convention and should not be interpreted as a deterministic confidence bonus.

\subsection{Multi-armed bandit experiments}
\label{subsec:mab-experiments}

We consider MAB problems with $K=2,4$, and $8$, at horizons up to $T=200$.
These experiments compare the SAB-based index policy with the benchmark policies.
For $K=2$, we also compute the grid-based MAB minimax policy by solving a discretization of the exact MAB formulation~\eqref{eq:mab-lp} with grid spacing $0.005$.
By symmetry, we restrict attention to $p_1 \geq p_2$, giving the optimization grid \(\Pgrid_{0.005}^{(2)} =
\left\{(p_1,p_2): p_i\in\{0,0.005,\ldots,1\},\ p_1\ge p_2\right\}\).
The direct formulation has \(2\binom{T+3}{4}=\Theta(T^4)\) decision variables.
Its solution provides the grid-based MAB minimax policy against which we compare the other policies.

Figure~\ref{fig:exp-mab-k2} compares the worst-case regret of the benchmark policies, the SAB-based index policy, and the grid-based MAB minimax policy for $K=2$.
As above, the reported values are the estimated worst-case regrets following the approach described in the beginning of the section.
The SAB-based index policy has lower regret than every benchmark policy at all five horizons, reducing regret relative to the best benchmark by \(11.2\%\)--\(13.2\%\).
Its regret is \(12.6\%\) higher than that of the grid-based MAB minimax policy at \(T=40\), with this gap decreasing to \(9.6\%\) at \(T=200\).

\begin{figure}[htbp]
  \centering
  \includegraphics[width=0.70\linewidth]{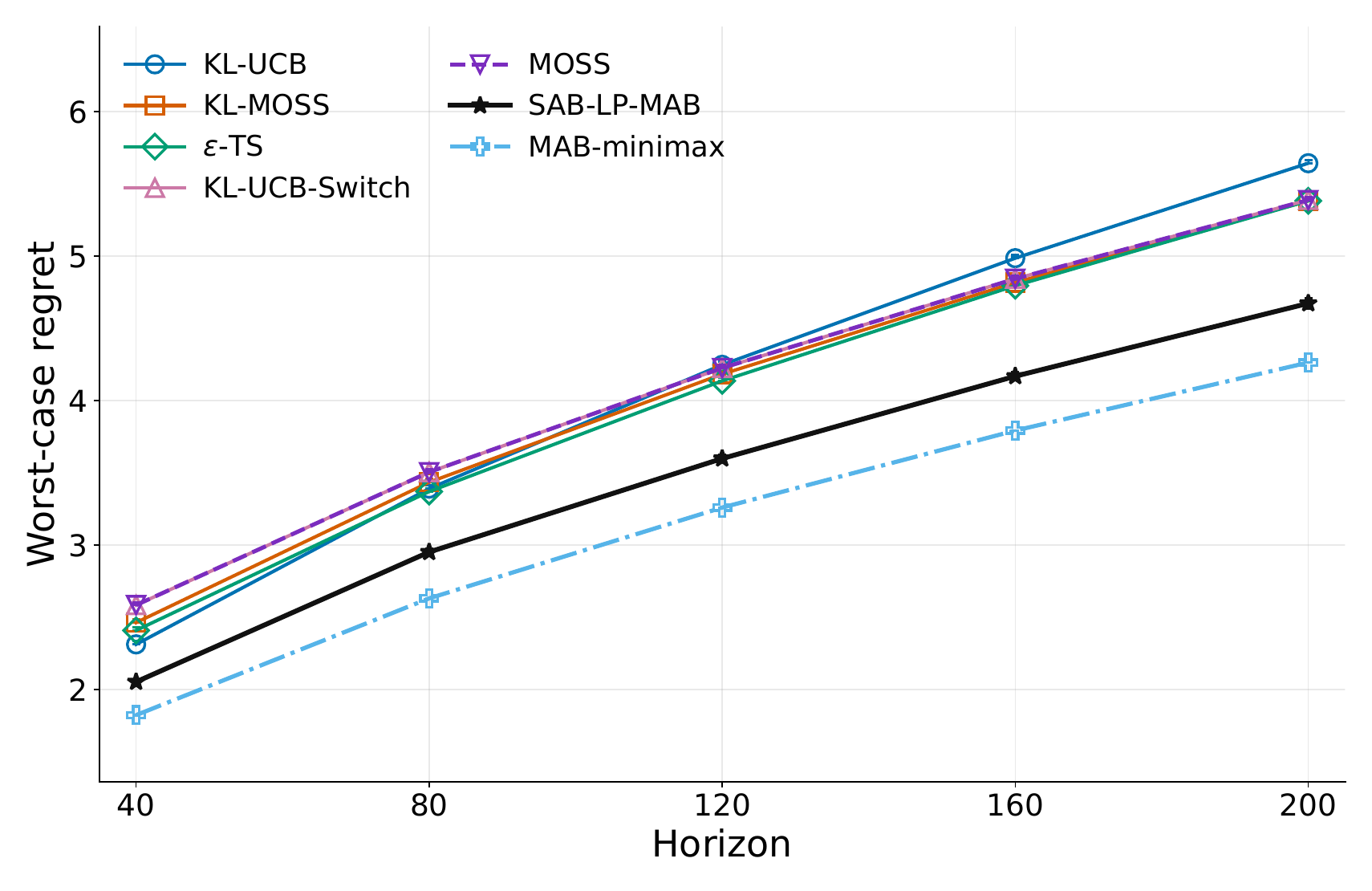}
  \caption{Estimated worst-case regret for \(K=2\). SAB-LP-MAB is compared with the benchmark policies and the grid-based MAB minimax policy computed on the \(0.005\)-spaced Bernoulli grid.}
  \label{fig:exp-mab-k2}
\end{figure}

Figures~\ref{fig:exp-mab-k4} and~\ref{fig:exp-mab-k8} show the same comparison for \(K=4\) and \(K=8\), without the grid-based MAB minimax policy.
For \(K=4\), the SAB-based index policy has regrets
\(4.629,6.557,8.076,9.374,10.480\).
KL-MOSS is the best benchmark, with regrets
\(5.619,7.862,9.538,10.856,12.117\), so the regret reduction ranges from \(17.6\%\) at \(T=40\) to \(13.5\%\) at \(T=200\).
For \(K=8\), the SAB-based index policy has regrets
\(8.035,11.250,13.850,15.911,17.895\), compared with
\(10.706,13.701,16.746,19.005,20.945\) for KL-MOSS.
The reduction is \(25.0\%\) at \(T=40\) and remains \(14.6\%\) at \(T=200\).
Across all 15 tested \((K,T)\) pairs, the SAB-based index policy has lower regret than every benchmark policy.
The advantage is largest for \(K=8\) at shorter horizons, where early exploration of each arm is particularly costly.
This pattern is consistent with the more conservative continuation decisions of the SAB policy at small pull counts, as shown in Figure~\ref{fig:exp-sab-boundary}.

We omit standard UCB and Tsallis-INF from Figures~\ref{fig:exp-mab-k2} and~\ref{fig:exp-mab-k48} because they perform substantially worse than the other methods. A comprehensive comparison including these policies is provided in Figure~\ref{fig:exp-mab-k48-full} in Appendix~\ref{sec:extendedexp}.

\begin{figure}[htbp]
  \centering
  \begin{subfigure}[t]{0.485\linewidth}
    \centering
    \includegraphics[width=\linewidth]{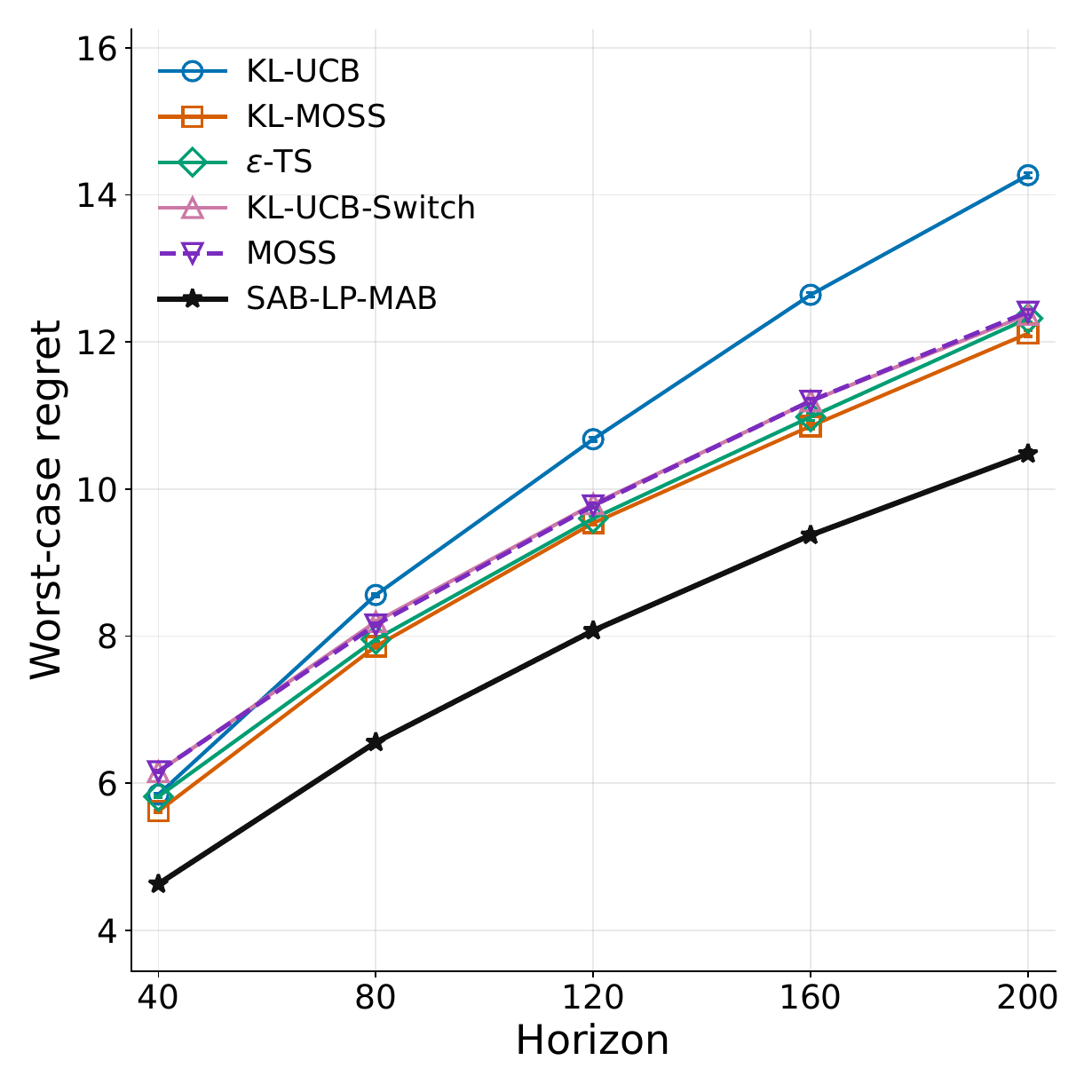}
    \caption{\(K=4\).}
    \label{fig:exp-mab-k4}
  \end{subfigure}\hfill
  \begin{subfigure}[t]{0.485\linewidth}
    \centering
    \includegraphics[width=\linewidth]{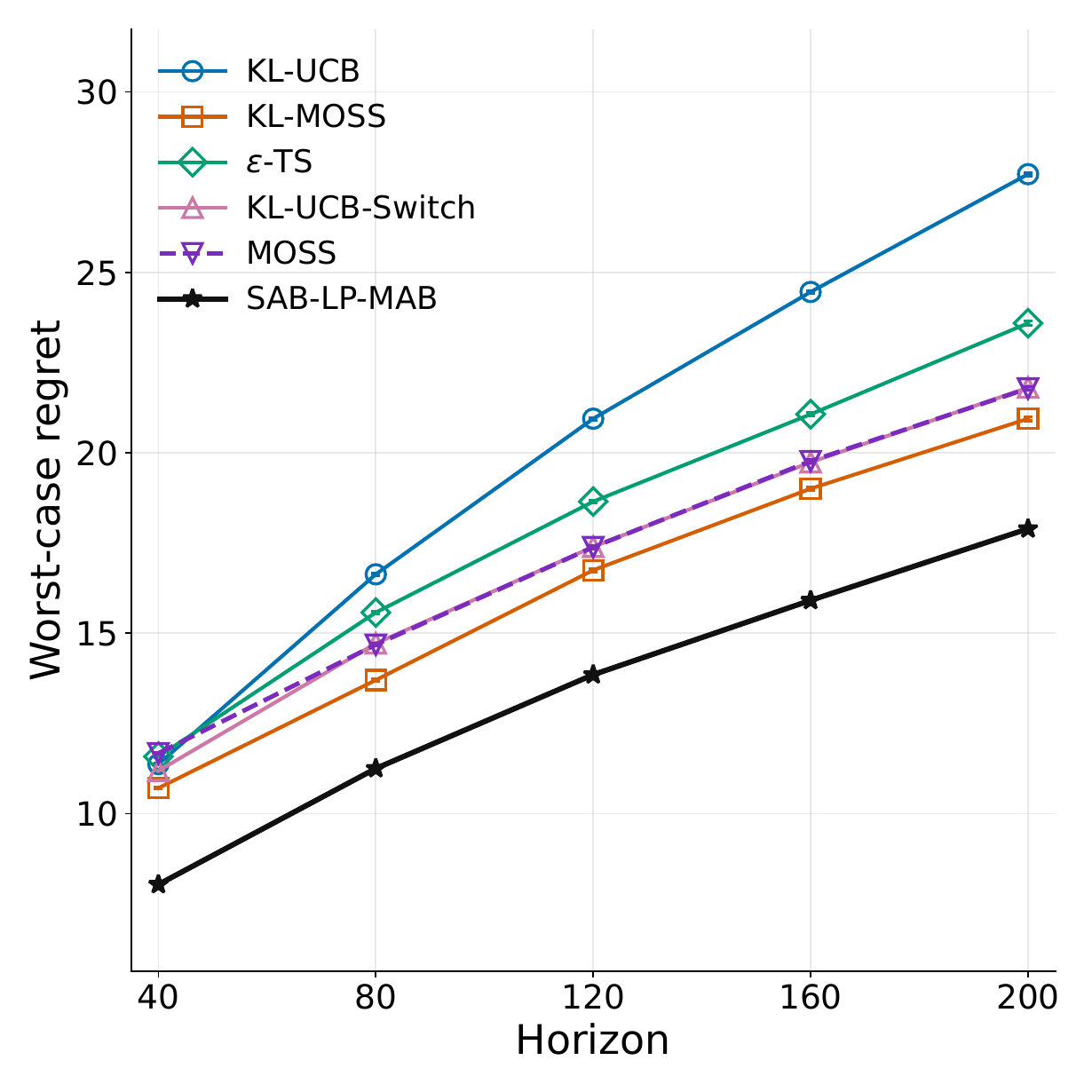}
    \caption{\(K=8\).}
    \label{fig:exp-mab-k8}
  \end{subfigure}
  \caption{Estimated worst-case regret over the evaluation grid for \(K=4\) and \(K=8\).
  SAB-LP-MAB uses the same untuned construction
  \(H=\lceil T/K \rceil\), \(G=2H\), and
  \(b=\sqrt{G}/4\) in both panels.}
  \label{fig:exp-mab-k48}
\end{figure}

%% file: main/appendix_proof.tex
\section{Proofs and auxiliary calculations}
\label{app:proofs}

This appendix proves the SAB and
index results used there and verifies the scalar regret constants.
The bounded-reward extension and the two parameter variants are treated
separately. The exact joint-MAB formulation is collected at the end;
it motivates the construction but is not used in the index-policy proof.

\subsection{SAB policy representation and the linear program}
\label{app:sab-representation}

\begin{proof}{Proof of Proposition~\ref{prop:count-state}}
By \citet[Lemma~4.1]{BradtJohnsonKarlin1956}, every admissible policy is equivalent, for every $p\in[0,1]$, to a stopping policy that takes the same number of unknown-arm observations, so we may assume that $\pi$ is a possibly randomized stopping policy. It remains to replace $\pi$ by a randomized stopping policy whose continuation probability depends only on $(n,s_n)$, without changing its regret for any $p$.

For a prescribed history $h_n$ with $s_n$ successes, let $w(h_n)$ be the
probability, over the policy's randomization, of taking the first $n$
observations, and let $q(h_n)$ be its conditional continuation
probability. Set $w(h_0)=1$ and $q(h_n)=0$ when $w(h_n)=0$.
These quantities do not depend on $p$, and
\[
\mathbb P_p(\pi\text{ reaches }h_n)
=w(h_n)p^{s_n}(1-p)^{n-s_n}.
\]
For \(0\le s\le n<G\), define
\begin{equation}
\label{eq:qcgns}
q^{c,G}_{n,s}
\coloneqq
\begin{cases}
\displaystyle
\frac{\sum_{h_n:s_n=s}w(h_n)q(h_n)}
     {\sum_{h_n:s_n=s}w(h_n)},
&\sum_{h_n:s_n=s}w(h_n)>0,\\[1ex]
0,&\text{otherwise}.
\end{cases}
\end{equation}
The common Bernoulli likelihood factor cancels, so these are
$p$-independent probabilities in $[0,1]$. For the original policy, write $E_{n,s}$ for the event that sampling
has reached $(n,s)$. The preceding definition gives
\[
\mathbb P_p(E_{n,s},\text{continue})
=q^{c,G}_{n,s}\mathbb P_p(E_{n,s}).
\]
The count-state policy defined by $q^{c,G}_{n,s}$ satisfies the same
identity. Under either policy, the state probabilities obey
\[
\mathbb P_p(E_{n+1,s})
=pq^{c,G}_{n,s-1}\mathbb P_p(E_{n,s-1})
 +(1-p)q^{c,G}_{n,s}\mathbb P_p(E_{n,s}),
\]
with the same initial value and zero terms outside the natural ranges.
Induction proves equality of the state and continuation probabilities,
including at $p=0,1$. Summing continuation probabilities gives equal
expected pull counts and hence equal regret by~\eqref{eq:sab-regret}.
\end{proof}

\begin{proof}{Proof of Theorem~\ref{thm:sab-lp}}
By the preceding reduction, it suffices to consider count-state stopping
policies. Given such a policy, set $x^{c,G}_{0,0}=q^{c,G}_{0,0}$ and
\[
x^{c,G}_{n,s}
=q^{c,G}_{n,s}\left(
\frac{s}{n}x^{c,G}_{n-1,s-1}
+\frac{n-s}{n}x^{c,G}_{n-1,s}\right),\qquad 1\le n<G.
\]
These entries are independent of $p$. Induction shows that they satisfy
the box and flow constraints in~\eqref{eq:sab-lp}. The state-action
recursion in Section~\ref{subsec:sab-lp} gives
\[
w^{c,G}_{n,s}(p)
=\binom ns p^s(1-p)^{n-s}x^{c,G}_{n,s},
\]
with the endpoint convention stated there. Summing over $(n,s)$ gives
the expected pull count and therefore the displayed regret constraints.

Conversely, recover a policy from any feasible array using
\eqref{eq:sab-recovery}, with that array in place of $x^{\star,c,G}$.
A positive denominator gives a ratio in $[0,1]$ by the flow inequality.
A zero denominator forces the numerator to vanish, so assigning zero
is consistent with the same recursion. Induction then recovers exactly
the prescribed state-action probabilities. The two constructions give
the same attainable regret functions and hence the same optimal value.

The feasible set of arrays is nonempty and compact. If two arrays differ
by at most $\varepsilon$ in every entry, their regrets differ by at most
$G\varepsilon$ for every $p$, because the binomial weights sum to one
at each count and $|p-c|\le1$. Thus the worst-case regret is continuous
on that compact set and attains its minimum. Taking $\eta$ equal to the
worst-case regret proves attainment of the linear program.

At $c=0$, the all-one array is feasible and has zero regret; at $c=1$,
the zero array has the same property. To show uniqueness of the
minimizing arrays at these endpoints, evaluate regret at $p=1/2$.
Zero regret forces the expected unknown-arm pull count to be $G$ at
$c=0$ and zero at $c=1$. Every binomial weight at $p=1/2$ is positive,
and their sum over all states and counts is $G$. Since every entry lies
in $[0,1]$, all entries must respectively be one or zero. The recovery
formula gives the claimed endpoint policies.
\end{proof}

\subsection{SAB regret bounds and finite-grid approximation}
\label{app:sab-bounds}

For $c\in[0,1]$ and $b>0$, define
\[
\tilde\tau_{c,b}
\coloneqq\inf\{n\ge1:\tilde s_n-cn\le-b\},
\qquad \inf\varnothing\coloneqq\infty.
\]
The reward-shortfall policy $\pi^b_{c,G}$ takes
$\tilde n^b_{c,G}\coloneqq\min\{\tilde\tau_{c,b},G\}$ unknown-arm
observations and then selects the known arm. This comparison policy
takes its first observation with certainty; the adjusted policy need
not do so.

\begin{lemma}[Exponential bounds]
\label{lem:exponential}
Let $\tilde z_1,\tilde z_2,\ldots$ be independent Bernoulli observations
with mean $p$, and let $\tilde s_n=\sum_{j=1}^n\tilde z_j$.
For $x>0$ and integers $n,H\ge1$,
\[
\mathbb P_p(\tilde s_n/n-p\ge x)\le e^{-2nx^2},
\qquad
\mathbb P_p\!\left(\sup_{n\ge H}(p-\tilde s_n/n)>x\right)
\le e^{-2Hx^2}.
\]
For $0\le c<p\le1$ and $b>0$,
\[
\mathbb P_p(\tilde\tau_{c,b}<\infty)\le e^{-8b(p-c)}.
\]
\end{lemma}

\begin{proof}{Proof}
For any real $\lambda$, a Bernoulli observation satisfies
$\mathbb E_p[e^{\lambda(\tilde z_1-p)}]\le e^{\lambda^2/8}$.
Indeed, the logarithm of the left-hand side and its first derivative
vanish at zero, while its second derivative is a variance under
exponential weighting and is at most $1/4$. Integrating this bound
twice proves the inequality. Independence and Markov's inequality,
with $\lambda=4x$, give the fixed-count bound.

We also use the following crossing argument. If a nonnegative
supermartingale starts at one, stop it at its first crossing of $A>0$
or at a fixed finite count. Its stopped expectation is at most one,
so the crossing probability is at most $1/A$. Letting the finite count
tend to infinity gives the same bound for an eventual crossing.

For the maximum over counts, apply this argument to
\[
\tilde m_n
=\exp\{4x(np-\tilde s_n)-2nx^2\},\qquad n\ge0.
\]
The moment bound makes this a nonnegative supermartingale. If
$p-\tilde s_n/n>x$ for some $n\ge H$, then
$\tilde m_n>e^{2Hx^2}$, proving the second assertion.

Finally, put $x=p-c>0$. The same moment bound gives
\[
\mathbb E_p e^{-8x(\tilde z_1-c)}
\le e^{-8x^2+(8x)^2/8}=1.
\]
Thus $\exp\{-8x(\tilde s_n-cn)\}$ is a nonnegative supermartingale
starting at one. At $\tilde\tau_{c,b}<\infty$ it reaches at least
$e^{8bx}$, which proves the last assertion.
\end{proof}

\begin{proof}{Proof of Proposition~\ref{cor:sab-upper}}
Use the reward-shortfall policy with tolerance $\beta>0$ and write
$\tilde n=\min\{\tilde\tau_{c,\beta},G\}$. Before a first crossing,
$\tilde s_n-cn>-\beta$, and one more observation can decrease this
quantity by at most $c$. If no crossing occurs by $G$, it remains
strictly above $-\beta$. Hence, in either case,
\[
\tilde s_{\tilde n}-c\tilde n>-\beta-c.
\]
The decision to take observation $j$ precedes its reward, so
\[
\mathbb E_p[\tilde s_{\tilde n}]
=\sum_{j=1}^G\mathbb E_p[
\tilde z_j\mathbf1\{\tilde n\ge j\}]
=p\mathbb E_p[\tilde n].
\]
For $p<c$, it follows that
$R_{c,G}(\pi^\beta_{c,G},p)<\beta+c$.
For $p>c$, put $x=p-c$ and apply Lemma~\ref{lem:exponential}:
\[
R_{c,G}(\pi^\beta_{c,G},p)
=x\mathbb E_p[G-\tilde n]
\le Gx\mathbb P_p(\tilde\tau_{c,\beta}<\infty)
\le Gxe^{-8\beta x}
\le\frac{G}{8e\beta}.
\]
At $p=c$, regret is zero. Choosing
\[
\beta=\frac{\sqrt{c^2+G/(2e)}-c}{2}>0
\]
equalizes $\beta+c$ and $G/(8e\beta)$ and gives
\[
V_G(c)
\le\frac{\sqrt{c^2+G/(2e)}+c}{2}
\le\sqrt{\frac{G}{8e}}+c.
\]
In particular,
$\sup_{c\in[0,1]}V_G(c)\le\sqrt{G/(8e)}+1$.
\end{proof}

\begin{proof}{Proof of Proposition~\ref{prop:sab-grid}}
Fix a stopping rule $\pi$, and write $m_\pi(p)$ for its expected
unknown-arm pull count. Couple Bernoulli sequences of means $p,p'$
using the same independent uniforms. Their first $n$ rewards differ
with probability at most $n|p-p'|$. Since the cumulative continuation
probability is in $[0,1]$ and is identical on identical histories,
\[
\left|\mathbb E_p[\rho^{c,G}(\tilde h_n)]
-\mathbb E_{p'}[\rho^{c,G}(\tilde h_n)]\right|
\le n|p-p'|.
\]
Summing over $0\le n<G$ yields
\[
|m_\pi(p)-m_\pi(p')|
\le\frac{G(G-1)}2|p-p'|.
\]
On each side of $c$, regret is the gap multiplied by the expected
number of inferior-arm selections. That count is at most $G$, and the
gap is at most one. Therefore
\[
|R_{c,G}(\pi,p)-R_{c,G}(\pi,p')|
\le\frac{G(G+1)}2|p-p'|
\]
when $p,p'$ lie on the same side of $c$. For opposite sides, apply
the two bounds through $c$, where regret is zero.

Every $p$ is within $1/(2M)$ of a point of $\mathcal P_M$.
Consequently, the full-interval worst-case regret of any feasible
policy is at most its grid regret plus $G(G+1)/(4M)$.
The finite-grid objective is continuous on the same compact feasible
set as the semi-infinite problem, and thus attains its minimum.
Applying the preceding bound to any minimizing array gives the
policy guarantee. Since the grid problem relaxes the full problem,
\[
V_{G,M}(c)\le V_G(c)
\le\sup_p R_{c,G}(\pi_{c,G,M},p)
\le V_{G,M}(c)+\frac{G(G+1)}{4M},
\]
which proves the value bound as well.

For the last assertion, hold the stopping rule fixed as $c$ varies;
in particular, do not reoptimize or change its continuation table.
Then $m_\pi(p)$ is unchanged, and
\[
R_{c,G}(\pi,p)=
\begin{cases}
(p-c)(G-m_\pi(p)),&c\le p,\\
(c-p)m_\pi(p),&c\ge p.
\end{cases}
\]
This is continuous and piecewise affine in $c$, with slopes in
$[-G,G]$. It proves~\eqref{eq:sab-c-lipschitz}. Taking the supremum
over $p$ and the infimum over the common set of stopping rules
preserves this uniform Lipschitz bound, giving the assertion for
$V_G$. The argument includes $c=0,1$.
\end{proof}

\subsection{Shared thresholds, adjusted policies, and indices}
\label{app:indices}

The endpoint claims and the nonempty sets defining the indices use
the analytical endpoint convention
\[
\rho^{0,G}(h_n)=1,\qquad \rho^{1,G}(h_n)=0,
\qquad 0\le n<G.
\]
It is automatic for the exact minimax family and is explicit in
Section~\ref{subsec:implementation}. For an arbitrary auxiliary
family, it must likewise be imposed: replace its endpoint policies
by always continuing at $c=0$ and always stopping at $c=1$.
This gives zero regret at both endpoints without changing any
interior policy or increasing the uniform SAB regret bound.
The index definitions in the main text are understood with this
convention.

\begin{proof}{Proof of Proposition~\ref{prop:shared-threshold}}
Condition on the entire reward sequence. Every continuation
probability is then fixed. With independent draws, taking observation
$n+1$ requires accepting all $n+1$ continuation decisions, with
conditional probability
$\prod_{j=0}^n q^{c,G}(\tilde h_j)=\rho^{c,G}(\tilde h_n)$.
With a shared uniform threshold independent of the reward sequence,
the continuation events are nested, so taking observation $n+1$
is equivalent to $\tilde u<\rho^{c,G}(\tilde h_n)$ and has the
same conditional probability.

A variable in $\{0,\ldots,G\}$ is determined in distribution by
these $G$ tail probabilities. The original stopping rule has the
same conditional survival probabilities by the chain rule, including
when its randomization is internally dependent: $q^{c,G}(h_j)$
is conditioned on having survived to $h_j$.
Thus both implementations preserve its pull-count law and regret.
For $p>c$ and $p<c$, respectively, the regret identity becomes
\[
\begin{aligned}
R_{c,G}(\pi_{c,G},p)
&=(p-c)\sum_{n=0}^{G-1}
\mathbb E_p[1-\rho^{c,G}(\tilde h_n)],&&p>c,\\
R_{c,G}(\pi_{c,G},p)
&=(c-p)\sum_{n=0}^{G-1}
\mathbb E_p[\rho^{c,G}(\tilde h_n)],&&p<c.
\end{aligned}
\]
These are precisely~\eqref{eq:continuation-regret}; at $p=c$
both sides are zero. For a count-state rule, comparison with the
normalized state-action representation also gives
\[
x^{c,G}_{n,s}
=\binom{n}{s}^{-1}\sum_{h_n:s_n=s}\rho^{c,G}(h_n).
\]
Thus $x^{c,G}_{n,s}$ averages the cumulative probabilities over
histories with the same terminal count; individual cumulative
probabilities may still depend on reward order.
\end{proof}

\begin{proof}{Proof of Proposition~\ref{prop:adjusted-regret}}
Fix the entire auxiliary family $\{\pi_{c,G}:c\in[0,1]\}$.
For each reward sequence, $\rho^{v,G}(h_n)$ is nonincreasing in $n$
for every $v$. Taking the supremum over $v\in[c,1]$ preserves this
property, as does multiplication by the nonincreasing indicator
$\chi^{c,b}(h_n)$. The adjusted probabilities are in $[0,1]$, so a
shared independent threshold realizes them as an admissible stopping
policy. Each decision depends only on the observed history.

For a fixed history, increasing $c$ shrinks the envelope's index set
and cannot increase $\chi^{c,b}$. The adjusted probabilities are
therefore nonincreasing in $c$ and Borel measurable. This conclusion
does not require a measurable selection of the original policies in
$c$. There are finitely many histories at each count, which also
ensures measurability in histories and thresholds. The endpoint
convention gives the stated values at $c=0,1$.

For $p>c$, necessarily $c<1$. The definition of the shortfall gate
then gives
$\{\chi^{c,b}(\tilde h_n)=0\}
=\{\tilde\tau_{c,b}\le n\}$, including the empty event at $n=0$.
Since $\bar\rho^{c,G}\ge\rho^{c,G}$,
\[
1-\rho^{\mathrm{adj},c,G,b}(\tilde h_n)
\le 1-\rho^{c,G}(\tilde h_n)
+\mathbf1\{\tilde\tau_{c,b}\le n\}.
\]
Multiply by $p-c$, sum, and apply the shared-threshold identity and
Lemma~\ref{lem:exponential} to obtain
\[
\begin{aligned}
R_{c,G}(\pi^{\mathrm{adj}}_{c,G,b},p)
&\le R_{c,G}(\pi_{c,G},p)
+G(p-c)\mathbb P_p(\tilde\tau_{c,b}<\infty)\\
&\le R_{c,G}(\pi_{c,G},p)
+G(p-c)e^{-8b(p-c)}.
\end{aligned}
\]
For $p<c$, the adjusted policy cannot take more observations than
$\pi^b_{c,G}$ on the same reward sequence, since its continuation
probability vanishes once the shortfall rule has stopped.
The bounded-stopping-time argument in the proof of
Proposition~\ref{cor:sab-upper} consequently gives
\[
R_{c,G}(\pi^{\mathrm{adj}}_{c,G,b},p)
\le(c-p)\mathbb E_p[\tilde n^b_{c,G}]<b+c.
\]
This comparison remains valid at $c=1$, where the adjusted endpoint
policy stops immediately. Dividing the two regret bounds by the
corresponding positive gaps proves~\eqref{eq:continuation-bounds}.
\end{proof}

\begin{proof}{Proof of Lemma~\ref{lem:index}}
Fix $h_n$ and $u\in[0,1)$. Both accepted sets contain $c=0$.
If $0\le c<j_G(h_n,u)$, some $v>c$ has
$\rho^{v,G}(h_n)>u$, and hence $\bar\rho^{c,G}(h_n)>u$.
Conversely, if $\bar\rho^{c,G}(h_n)>u$, some $v\ge c$ has
$\rho^{v,G}(h_n)>u$, so $c\le j_G(h_n,u)$.
Thus the monotone envelope does not change the supremum index.

For $c<1$, the shortfall gate further requires $c<\ell_b(h_n)$;
this also holds at $n=0$, since $\ell_b(h_0)=1$.
Every $c$ strictly below
$\min\{j_G(h_n,u),\ell_b(h_n)\}$ is therefore accepted, and no $c$
strictly above it is accepted. The endpoint $c=1$ is rejected by
convention, while $c=0$ is accepted. This proves
\eqref{eq:index-minimum}, without requiring attainment of the supremum.

If $j^{\mathrm{adj}}_{G,b}(h_n,u)<c$, acceptance at $c$ is impossible
by the definition of a supremum. If
$j^{\mathrm{adj}}_{G,b}(h_n,u)>c$, some $v>c$ is accepted, and
monotonicity in $c$ implies
\[
\rho^{\mathrm{adj},c,G,b}(h_n)\ge
\rho^{\mathrm{adj},v,G,b}(h_n)>u.
\]
Its contrapositive gives the stated weak converse.
Finally, the accepted sets shrink both when observations are added
and when $u$ increases. The adjusted index is nonincreasing in the
observation count, and its monotonicity in $u$ makes it Borel measurable.
\end{proof}

\subsection{Evaluation of the regret constants}
\label{subsec:lambda-evaluation}
\label{app:numerical}

We first bound the functional $\Lambda_0$ in~\eqref{eq:tail-integral}
for the tail function~\eqref{eq:phi}. The post-switch scalar
optimization is verified afterward. 

\begin{lemma}[Bounds for the scalar tail functional]
\label{lem:quantity-kappa-G}
Let $\kappa_0=1/\sqrt{8e}$. The function $\Lambda_0$ is finite,
increasing, and continuously differentiable on $(0,\infty)$, with
\[
\Lambda_0(\kappa_0)<2.671,
\qquad
0<\Lambda_0'(\kappa)\le\sqrt2(2+3\log3).
\]
Consequently, if $0<\kappa_G\le\kappa_0+3/(2\sqrt G)$, then
\[
\Lambda_0(\kappa_G)
\le2.671+\frac{3\sqrt2}{2\sqrt G}(2+3\log3).
\]
\end{lemma}

\begin{proof}{Proof}
\textit{Step 1: locate the maximizing gap.}
Write
\[
\begin{gathered}
A=\sqrt2,\qquad B=2,\qquad\alpha=2\sqrt2,
\qquad a_\kappa=A\kappa,\\
g(y)=Be^{-\alpha y}+e^{-2y^2},
\qquad f_\kappa(y)=\frac{a_\kappa}{y}+g(y).
\end{gathered}
\]
The function $f_\kappa$ is strictly decreasing from infinity to zero,
so it has a unique positive root of
\begin{equation}
\label{eq:lambda-cutoff}
f_\kappa(r_\kappa)=1.
\end{equation}
Thus $\Phi_\kappa=1$ on $(-\infty,r_\kappa]$ and equals $f_\kappa$
afterward. Since $g(1/2)=2e^{-\sqrt2}+e^{-1/2}>1$, we have
$r_\kappa>1/2$. Moreover, for $y\ge1/2$,
\[
-\frac{d}{dy}[yf_\kappa(y)]
=B(\alpha y-1)e^{-\alpha y}+(4y^2-1)e^{-2y^2}>0.
\]
With $h_\kappa=-f_\kappa'>0$, this implies
\begin{equation}
\label{eq:lambda-tail-domination}
h_\kappa(y)\ge\frac{f_\kappa(y)}{y},
\qquad f_\kappa(y)\le\frac{r_\kappa}{y},
\qquad y\ge r_\kappa.
\end{equation}
Abbreviate $r=r_\kappa$ and $a=a_\kappa$ when $\kappa$ is fixed.

The last integral in $J_{\Phi_\kappa}(d)$ is
$2\int_0^\infty\Phi_\kappa(t+d)\Phi_\kappa(t)\,dt$.
The function $\Phi_\kappa$ is continuous and locally absolutely
continuous; its tail and derivative are respectively $O(1/y)$ and
$O(1/y^2)$. These bounds justify differentiation and integration by
parts, giving, for $d\ne2r$,
\begin{equation}
\label{eq:lambda-J-derivative}
J_{\Phi_\kappa}'(d)
=\frac d2\Phi_\kappa'(d/2)
+2\int_r^\infty f_\kappa(t+d)h_\kappa(t)\,dt.
\end{equation}
For $d<2r$, the first term vanishes and the integral is positive.
For $d>2r$, put $x=d/2>r$ and use
$\int_r^\infty h_\kappa(t)\,dt=1$. Monotonicity of $f_\kappa$ and
$yf_\kappa(y)$ gives
\[
\begin{aligned}
J_{\Phi_\kappa}'(d)
&\le-xh_\kappa(x)+2f_\kappa(r+d)\\
&\le-f_\kappa(x)+\frac{2x}{r+d}f_\kappa(x)
=-\frac{r}{r+d}f_\kappa(x)<0.
\end{aligned}
\]
Continuity therefore shows that the unique maximizer is $d=2r$.
Substituting it and splitting the convolution at $t=r$ yields
\begin{equation}
\label{eq:lambda-representation}
\Lambda_0(\kappa)
=2r+2\int_r^{3r}f_\kappa(y)\,dy
+2\int_r^\infty f_\kappa(y)f_\kappa(y+2r)\,dy.
\end{equation}
In particular, $\Lambda_0$ is finite.

\par\medskip\noindent\textit{Step 2: control sensitivity to $\kappa$.}
The implicit function theorem gives
\[
r_\kappa'=\frac{A}{r_\kappa h_\kappa(r_\kappa)}\le A.
\]
The representation~\eqref{eq:lambda-representation} is continuously
differentiable: on every compact $\kappa$ interval, the integrand
of its improper integral and its parameter derivatives have tails
bounded uniformly by constant multiples of $1/y^2$.
Define
\[
\mathcal C_\kappa
=\int_r^\infty
\left[\frac{f_\kappa(y+2r)}{y}
+\frac{f_\kappa(y)}{y+2r}\right]dy,
\qquad
\mathcal B_\kappa
=\int_r^\infty h_\kappa(y)f_\kappa(y+2r)\,dy.
\]
The direct parameter derivatives in~\eqref{eq:lambda-representation}
contribute $2A(\log3+\mathcal C_\kappa)$. The moving endpoints and
shift $2r$ contribute
\[
4r_\kappa'\left[f_\kappa(3r)
-\int_r^\infty f_\kappa(y)h_\kappa(y+2r)\,dy\right]
=4r_\kappa'\mathcal B_\kappa,
\]
where the equality follows by integrating the derivative of
$f_\kappa(y)f_\kappa(y+2r)$. Hence
\[
\Lambda_0'(\kappa)
=2A(\log3+\mathcal C_\kappa)
+4r_\kappa'\mathcal B_\kappa>0.
\]
By~\eqref{eq:lambda-tail-domination},
\[
\mathcal C_\kappa
\le2r\int_r^\infty\frac{dy}{y(y+2r)}=\log3.
\]
Also, for $y\ge r$,
\[
f_\kappa(y+2r)
\le\frac{y}{y+2r}f_\kappa(y)
\le\frac{f_\kappa(y)}{1+2f_\kappa(y)}.
\]
Substitute $u=f_\kappa(y)$ to get
\[
\mathcal B_\kappa
\le\int_0^1\frac{u}{1+2u}\,du
=\frac{2-\log3}{4}.
\]
Combining these bounds proves
\begin{equation}
\label{eq:lambda-uniform-derivative}
0<\Lambda_0'(\kappa)
\le4A\log3+A(2-\log3)
=A(2+3\log3)=\sqrt2(2+3\log3).
\end{equation}
The argument in Steps 1--2 applies equally to other positive
coefficients $A,B,\alpha$ whenever $yf_\kappa(y)$ is decreasing
on $[r_\kappa,\infty)$; this observation will be used for the
parameter variants below.

\par\medskip\noindent\textit{Step 3: certify the baseline value.}
Define
\[
\mathcal G(x)=\int_x^\infty g(y)\,dy
=\frac{B}{\alpha}e^{-\alpha x}
+\frac{\sqrt\pi}{2\sqrt2}\operatorname{erfc}(\sqrt2 x)
\]
and
\[
\mathcal R_\kappa(y)
=\frac{a}{y}g(y+2r)
+\frac{a}{y+2r}g(y)
+g(y)g(y+2r).
\]
The algebraic part of the product integral is explicit:
$\int_r^\infty[y(y+2r)]^{-1}\,dy=\log3/(2r)$.
Thus~\eqref{eq:lambda-representation} becomes
\begin{equation}
\label{eq:lambda-fast-evaluation}
\Lambda_0(\kappa)
=2r+\left(2a+\frac{a^2}{r}\right)\log3
+2[\mathcal G(r)-\mathcal G(3r)]
+2\int_r^\infty\mathcal R_\kappa(y)\,dy.
\end{equation}
At $\kappa_0=1/\sqrt{8e}$, $a=1/(2\sqrt e)$.
Interval evaluation of $f_{\kappa_0}-1$ at the two endpoints gives
\[
r_-=0.735965204197677442
<r_{\kappa_0}<
r_+=0.735965204197677443.
\]
Then, numerical computation of \eqref{eq:lambda-fast-evaluation} shows that
$$
\Lambda_0(\kappa) \le 2 r_+ + \left(2a+\frac{a^2}{r_-}\right)\log3 + 2[\mathcal G(r_-)-\mathcal G(3r_+)] + 2\int_{r_-}^\infty\mathcal R_\kappa(y)\,dy <2.670963621<2.671.
$$
Finally, integrating~\eqref{eq:lambda-uniform-derivative} gives
\[
\Lambda_0(\kappa)
\le2.671+\sqrt2(2+3\log3)(\kappa-\kappa_0),
\qquad \kappa\ge\kappa_0.
\]
Monotonicity and the bound
$\kappa_G\le\kappa_0+3/(2\sqrt G)$ prove the lemma, including when
$\kappa_G<\kappa_0$.
\end{proof}

\subsubsection{Evaluation of the post-switch term}
\label{app:combination}

Recall
\[
\gamma_\star^2e^{\gamma_\star^2/2}=2,
\qquad
\gamma_0=\frac{\gamma_\star(\gamma_\star^2+1)}
{\gamma_\star^2+2}.
\]
The defining function for $\gamma_\star$ is strictly increasing from
zero to infinity, so the positive solution is unique, with
$0<\gamma_0<\gamma_\star$. For $u>\gamma_0$, let
\[
\psi(u)=\frac{2(u-\gamma_0)}{u^2}e^{-u^2/2}.
\]
Its logarithmic derivative is
\[
\frac{\psi'(u)}{\psi(u)}
=\frac1{u-\gamma_0}-\frac2u-u
=\frac{\gamma_0(u^2+2)-u(u^2+1)}{u(u-\gamma_0)}.
\]
The function $u(u^2+1)/(u^2+2)$ has positive derivative
$(u^4+5u^2+2)/(u^2+2)^2$. Thus $\psi$ increases up to
$\gamma_\star$ and decreases afterward. The defining equation gives
\begin{equation}
\label{eq:post-switch-scalar-maximum}
\sup_{u>\gamma_0}\frac{2(u-\gamma_0)}{u^2}e^{-u^2/2}
=\frac{2(\gamma_\star-\gamma_0)}{\gamma_\star^2}
 e^{-\gamma_\star^2/2}
=\gamma_\star-\gamma_0.
\end{equation}
For $\delta=\gamma_0/\sqrt H$, the substitution $u=\Delta\sqrt H$
therefore gives
\[
\sup_{\Delta>\delta}
\left(1-\frac\delta\Delta\right)\frac2\Delta e^{-H\Delta^2/2}
=(\gamma_\star-\gamma_0)\sqrt H.
\]
The other part is bounded by
$\sup_{\Delta>0}\Delta e^{-H\Delta^2/2}=1/\sqrt{eH}$.
The supremum of the sum is at most the sum of these suprema, so
\begin{equation}
\label{eq:post-switch-scalar-bound}
\sup_{\Delta>\delta}
\left(1-\frac\delta\Delta\right)e^{-H\Delta^2/2}
\left(\Delta+\frac2\Delta\right)
\le(\gamma_\star-\gamma_0)\sqrt H+\frac1{\sqrt{eH}}.
\end{equation}
Interval evaluation gives
\[
1.065028910790485417<\gamma_\star<1.065028910790485418.
\]
In particular, the constants used in the main-text combination satisfy
\[
\begin{aligned}
2+\frac{1/\sqrt2+\gamma_\star-\gamma_0}{2\sqrt2}
+\frac1{\sqrt{2e}}&<2.799020<\frac{14}{5},\\
2.671+\frac1{\sqrt2}+\gamma_\star&<4.45,\\
\frac{14}{5}+\frac32(2+3\log3)&<10.75.
\end{aligned}
\]

\subsection{The exact joint-MAB formulation and its size}
\label{subsec:mab-exact}

We use $T$ for the joint-MAB horizon, reserving $G$ for the auxiliary
SAB horizon. Before a selection, let $\mathbf n=(n_1,\ldots,n_K)$ and
$\mathbf s=(s_1,\ldots,s_K)$ be the per-arm pull and success counts.
The decision-state set is
\[
\mathcal S_{K,T}
=\{(\mathbf n,\mathbf s): n_k\in\{0,1,\ldots\},\ 
0\le s_k\le n_k,\ |\mathbf n|\le T-1\},
\qquad |\mathbf n|=\sum_kn_k.
\]
A count-state policy selects arm $k$ with probability
$q_{\mathbf n,\mathbf s}(k)$, where $\sum_kq_{\mathbf n,\mathbf s}(k)=1$.

\paragraph{Count-state reduction.}
For any ordered action-reward history $\mathfrak h$ with counts
$(\mathbf n,\mathbf s)$, its probability factorizes as
\[
\mathbb P_{\mathbf p}(\mathfrak h)
=w(\mathfrak h)\prod_{k=1}^Kp_k^{s_k}(1-p_k)^{n_k-s_k},
\]
where $w(\mathfrak h)$ is the probability of its prescribed actions
under the prescribed rewards and does not depend on $\mathbf p$.
Average the next-action probabilities over histories with the same
counts, using these weights. If the total weight is zero, choose any
fixed action distribution. The common likelihood factor cancels,
and induction over the total count proves that the resulting
count-state policy preserves every state-action probability and every
expected arm pull count, for every $\mathbf p$. It therefore has the
same regret. This is the multi-arm version of the reduction proved above.

\paragraph{State-action formulation.}
Let
\[
\beta_{\mathbf p}(\mathbf n,\mathbf s)
=\prod_{k=1}^K\binom{n_k}{s_k}p_k^{s_k}(1-p_k)^{n_k-s_k}.
\]
Define $\mathbf p$-independent normalized masses
$y_{\mathbf n,\mathbf s,k}$ by
\[
\mathbb P_{\mathbf p}(\text{reach }(\mathbf n,\mathbf s)
\text{ and select }k)
=\beta_{\mathbf p}(\mathbf n,\mathbf s)y_{\mathbf n,\mathbf s,k}.
\]
The independence from $\mathbf p$ follows from the preceding
likelihood factorization. At boundary means, use continuous extension.
With $e_k$ the $k$th unit vector, flow conservation gives
\begin{equation}
\label{eq:mab-flows}
\begin{aligned}
\sum_{k=1}^Ky_{\mathbf0,\mathbf0,k}&=1,\\
\sum_{k=1}^Ky_{\mathbf n,\mathbf s,k}
&=\sum_{k:n_k\ge1}\left[
\frac{s_k}{n_k}y_{\mathbf n-e_k,\mathbf s-e_k,k}
+\frac{n_k-s_k}{n_k}y_{\mathbf n-e_k,\mathbf s,k}\right],
\quad 1\le|\mathbf n|\le T-1.
\end{aligned}
\end{equation}
Terms outside the natural index ranges are zero. Conversely, any
nonnegative array satisfying these equations defines a policy by
\[
q_{\mathbf n,\mathbf s}(k)
=\frac{y_{\mathbf n,\mathbf s,k}}
{\sum_{k'}y_{\mathbf n,\mathbf s,k'}}
\]
when the denominator is positive, with an arbitrary fixed action
distribution otherwise. Induction verifies the prescribed masses.
Since the expected number of pulls of arm $k$ is the sum of its
state-action probabilities, the exact formulation is
\begin{equation}
\label{eq:mab-lp}
\begin{aligned}
V^{\mathrm{MAB}}_{K,T}=\min_{\eta,y}\quad&\eta\\
\text{subject to}\quad
&\eta\ge\sum_{k=1}^K(p^\star-p_k)
\sum_{(\mathbf n,\mathbf s)\in\mathcal S_{K,T}}
\beta_{\mathbf p}(\mathbf n,\mathbf s)y_{\mathbf n,\mathbf s,k},
&&\mathbf p\in[0,1]^K,\\
&\text{the flow equations~\eqref{eq:mab-flows}},
\qquad y_{\mathbf n,\mathbf s,k}\ge0.
\end{aligned}
\end{equation}
The feasible arrays form a compact set: under
$\mathbf p=(1/2,\ldots,1/2)$, every $\beta_{\mathbf p}$ is positive,
and the corresponding state-action probability is at most one.
The feasible set is closed and nonempty, and the worst-case regret
is continuous on it, so the minimum is attained.

The SAB formulation uses the same likelihood normalization, but
also includes a known alternative and the option to stop. These
additional features yield its flow inequalities. Simply setting
$K=1$ in~\eqref{eq:mab-lp} does not give the SAB problem: a single
unknown arm without a known alternative has zero regret.

\paragraph{Size of the exact program.}
The generating function
$\sum_{\mathbf n}\prod_k(n_k+1)z^{|\mathbf n|}=(1-z)^{-2K}$ gives
\[
|\mathcal S_{K,T}|
=\sum_{|\mathbf n|\le T-1}\prod_k(n_k+1)
=\binom{T+2K-1}{2K}.
\]
There are $K|\mathcal S_{K,T}|$ policy variables, one additional
variable $\eta$, and $|\mathcal S_{K,T}|$ flow equalities.
For fixed $K$, the policy-variable and flow-constraint counts grow
as $\Theta(T^{2K})$. The regret constraints are still indexed by the
continuum $\mathbf p\in[0,1]^K$; the exact program is semi-infinite.
For example, the policy-variable counts are $83{,}834{,}250{,}500$
for $K=2,T=1000$ and $4{,}828{,}032{,}300$ for $K=3,T=100$.

\section{Comprehensive Comparison in MAB Setting}
\label{sec:extendedexp}
In Figure~\ref{fig:exp-mab-k48-full}, we include the results of standard UCB and Tsallis-INF for a more comprehensive comparison. The Tsallis-INF policy follows~\cite{ZimmertSeldin2021}, sampling arms according to Tsallis-regularized cumulative estimated losses with symmetric regularization and $\alpha=1/2$.
We implement both the importance-weighted (IW) and reduced-variance (RV) variants, using losses $1-X_t$ and learning rates $\eta_t=2/\sqrt{t}$ and $\eta_t=4/\sqrt{t}$, respectively, as specified in \citep[Theorem~1]{ZimmertSeldin2021}.

\begin{figure}[htbp]
  \centering
  \begin{subfigure}[t]{0.485\linewidth}
    \centering
    \includegraphics[width=\linewidth]{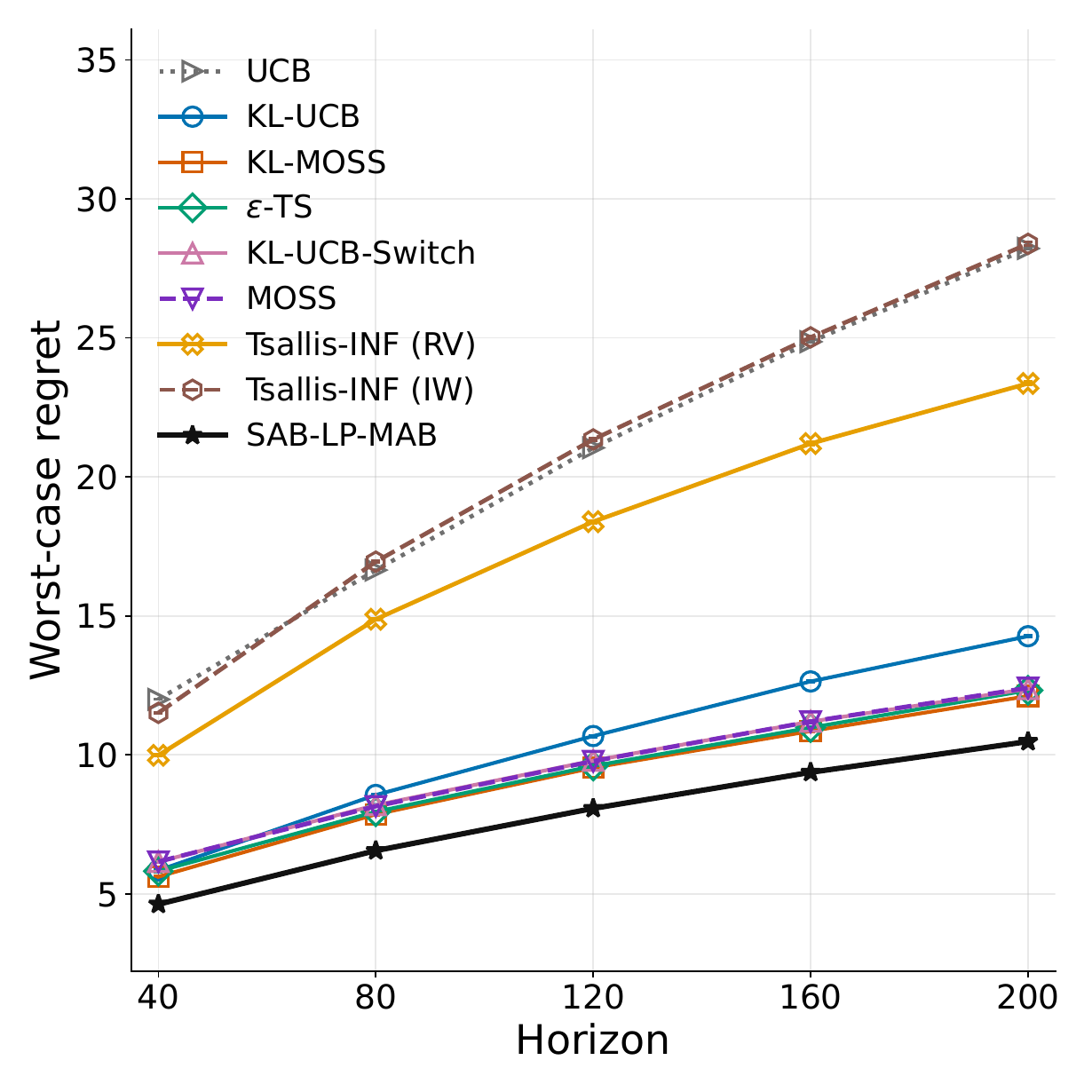}
    \caption{\(K=4\).}
    \label{fig:exp-mab-k4-full}
  \end{subfigure}\hfill
  \begin{subfigure}[t]{0.485\linewidth}
    \centering
    \includegraphics[width=\linewidth]{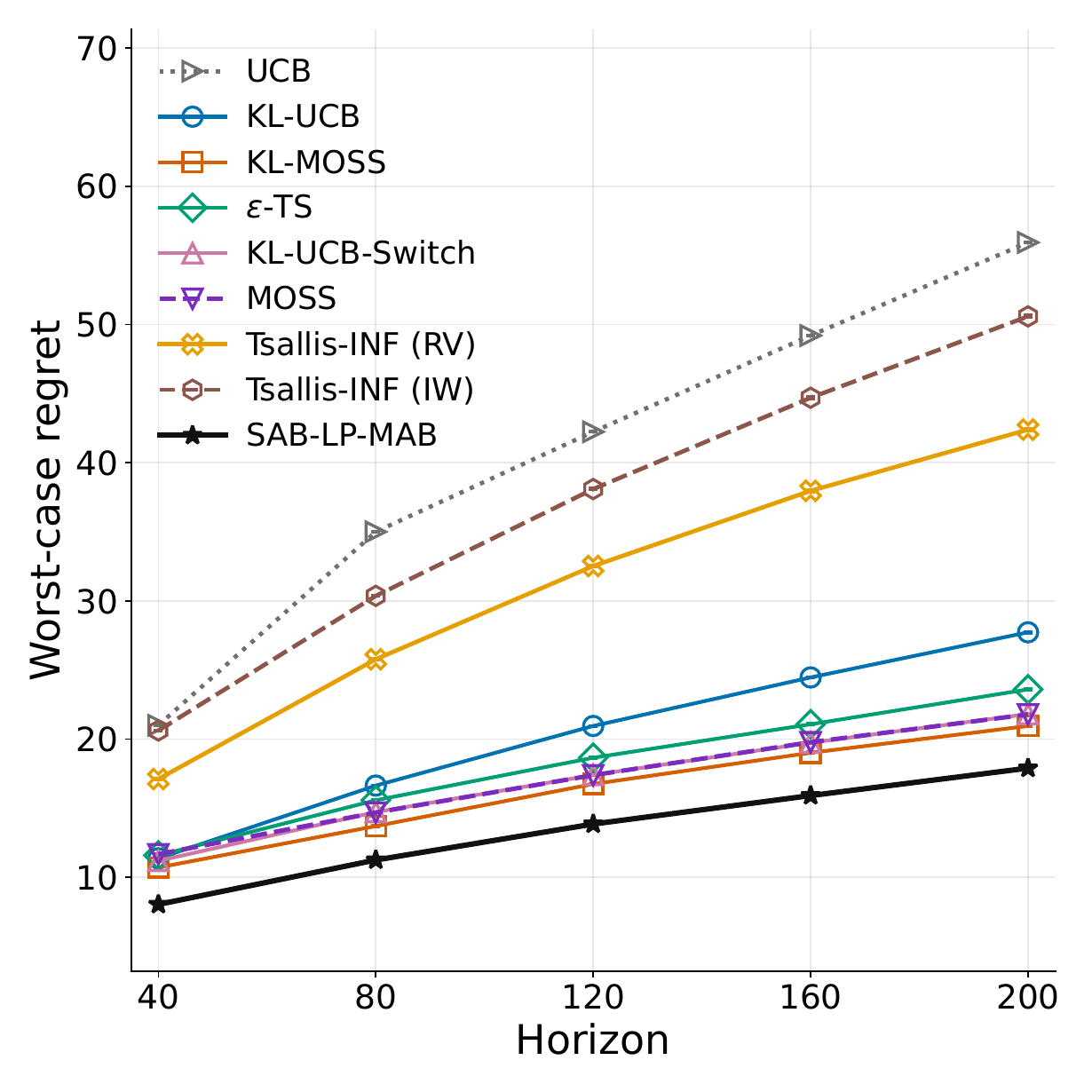}
    \caption{\(K=8\).}
    \label{fig:exp-mab-k8-full}
  \end{subfigure}
  \caption{Estimated worst-case regret over the evaluation grid for \(K=4\) and \(K=8\).
  SAB-LP-MAB uses the same untuned construction
  \(H=\lceil T/K \rceil\), \(G=2H\), and
  \(b=\sqrt{G}/4\) in both panels.}
  \label{fig:exp-mab-k48-full}
\end{figure}